\documentclass{article}
\usepackage[OT1]{fontenc}

\usepackage[table]{xcolor}
\usepackage[preprint]{neurips_2024}
\usepackage{natbib}
\usepackage{amsmath,amssymb,amsfonts,amsthm,mathtools,bm}
\usepackage{booktabs,multirow,array,graphicx,subcaption}
\graphicspath{{./}{../}}
\usepackage{microtype,url,hyperref}
\usepackage{placeins}
\usepackage{enumitem}
\usepackage{algorithm}
\usepackage{algpseudocode}
\usepackage{tikz}
\definecolor{customblue}{HTML}{b9dcb6}  
\usetikzlibrary{arrows.meta,positioning}

\definecolor{dblue}{HTML}{174A7E}
\usepackage{tcolorbox}
\tcbuselibrary{skins,breakable}
\hypersetup{colorlinks=true,linkcolor=dblue,citecolor=dblue,urlcolor=dblue}

\newtheorem{theorem}{Theorem}
\newtheorem{proposition}{Proposition}

\newtheorem{lemma}{Lemma}
\theoremstyle{definition}
\newtheorem{definition}{Definition}

\newtheorem{assumption}{Assumption}

\theoremstyle{remark}

\newcommand{\cC}{\mathsf{C}}                 
\newcommand{\cCz}{\mathsf{C}_0}              
\newcommand{\cD}{\mathsf{D}}                 
\newcommand{\incl}{K}                        
\newcommand{\cK}{\mathcal{K}}                
\newcommand{\cL}{\mathcal{L}}                
\newcommand{\Fz}{F_0}                        
\newcommand{\Ext}{\mathrm{Ext}}
\newcommand{\ExtN}{\mathrm{Ext}_N}
\newcommand{\Adm}{\mathrm{Adm}}
\newcommand{\Canon}{\mathrm{Canon}}
\newcommand{\Lan}{\mathrm{Lan}}
\newcommand{\Ran}{\mathrm{Ran}}
\newcommand{\Kan}{\mathrm{Kan}}
\newcommand{\Lans}{\Lan^{*}}
\newcommand{\Rans}{\Ran^{*}}
\newcommand{\FinSet}{\mathbf{FinSet}}
\newcommand{\In}{\mathrm{In}}
\newcommand{\Out}{\mathrm{Out}}
\newcommand{\Ob}{\mathrm{Ob}}
\newcommand{\Hom}{\mathrm{Hom}}
\newcommand{\Aut}{\mathrm{Aut}}
\newcommand{\Sym}{\mathrm{Sym}}
\newcommand{\chance}{\mathrm{chance}}
\newcommand{\KD}{\mathrm{KD}}
\newcommand{\DC}{\mathrm{DC}}
\newcommand{\E}{\mathbb{E}}
\newcommand{\Pp}{\mathbb{P}}
\newcommand{\N}{\mathbb{N}}
\newcommand{\id}{\mathrm{id}}

\hypersetup{colorlinks=true,linkcolor=red!70!black,linktocpage=false,citebordercolor=blue!70!black,citecolor=blue!70!black,anchorcolor=blue!70!black}

\title{When the Canonical Completion Is Wrong:\\
Formalizing and Measuring the Jump
in LLMs}

\author{%
\textbf{Dai Shi}\,\textsuperscript{1} \quad \textbf{Xiaoyu Li}\,\textsuperscript{2} \quad
\textbf{Jos\'e Miguel Hern\'andez-Lobato}\,\textsuperscript{1} \\
{\footnotesize \textsuperscript{1}University of Cambridge \quad \textsuperscript{2}University of New South Wales}
}

\begin{document}

\maketitle

\vspace{-15pt}
\raggedbottom
\begin{abstract}
\vspace{-5pt}
Whether large language models (LLMs) can perform the abductive leap from
evidence to a new system of axioms, commonly referred to as a jump, has
recently attracted considerable debate. A prominent position holds that LLMs
are structurally incapable of such jumps, while recent studies challenge both
its mechanism and empirical evidence. One of the main reasons why the debate
remains open is the difficulty of defining the jump precisely enough to test it.
In this paper, we attempt to develop a formal account of the jump in four steps
and measure the second. These steps ask what the default completion of partial
data is, when the constraints exclude it, whether the new structure agrees
with later observations, and how successive jumps compound. We define a
\emph{jump instance} as a finite extension problem whose constraints exclude
the canonical completions given by the Kan extensions and leave one correct
completion up to renaming. In this setting, a model with a canonical default
performs the second step by producing the correct completion under the constraints.
We evaluate fourteen models across three certified families. The canonical
completion returns once in $13{,}300$ constrained answers across all runs.
Several calibrated models also give the correct completion reliably, including
three API models that solve $159$ of $162$ primary chain trials, suggesting
that they can jump at this step. We further formalize the third and fourth
steps, whose empirical evaluation remains future work. We hope our work paves
the path for formalizing and measuring the full jump in the future. The code
of the paper is available at \url{https://github.com/EEthanShi/kan-jump-test}.
\end{abstract}

\begin{figure}[H]
\centering
\scalebox{0.80}{
\includegraphics[width=\textwidth]{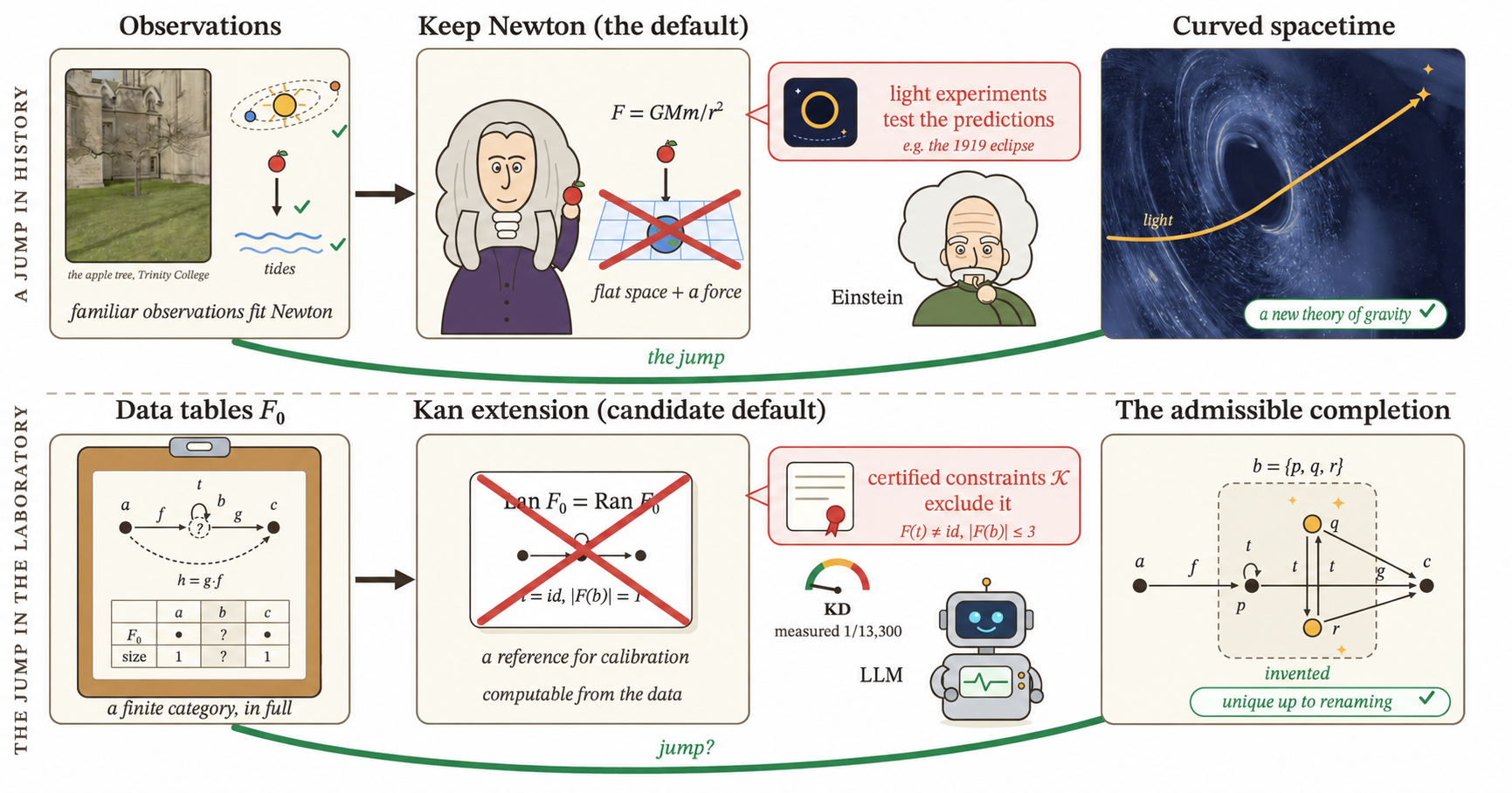}}
\refstepcounter{figure}\label{fig:overview}
\caption*{A historical analogy and the laboratory test. Top: observations of light
deflection test the predictions of Newtonian gravity and general relativity.
Bottom: partial tables leave $b$ unspecified. The Kan extensions supply a
singleton reference completion, while the stated constraints select a
three-element completion, unique up to renaming. Calibration checks whether
the reference describes the model's default; the constrained task tests
whether the model constructs the admissible completion.}
\end{figure}

\section{Introduction}\label{sec:intro}

Can a large language model (LLM) make the kind of leap by which Einstein
passed from Newtonian gravity to the curved spacetime of general relativity?
Such a leap, commonly referred to as a jump, is the abductive move from
evidence to a new system of axioms \citep{peirce1934collected}.
\citet{zahavy2026jump} raised this question and answered it negatively,
arguing that LLMs have mastered induction and are increasingly mastering
deduction, yet remain structurally incapable of the jump. Its proposed
mechanism has since been challenged \citep{farmer2026abduction}, while
earlier work on LLM-guided algorithm discovery provides related empirical
evidence \citep{novikov2025alphaevolve}. However, the debate remains
difficult to settle without a shared formal definition of the jump and
a measure by which the competing claims could be tested.

In this paper, we define a jump in LLMs in four steps and evaluate the first
two via category theory. We do this by testing whether a model can complete
data against its usual default, in a setting whose constraints exclude the
answer that comes naturally to it. These four steps ask what the default
completion of the data is, when the constraints exclude that default,
whether the new structure agrees with later observations, and how successive
jumps compound. The first two steps together ask
\begin{tcolorbox}[enhanced, colback=blue!5, colframe=blue!35!black,
boxrule=0.6pt, arc=1.5pt, left=7pt, right=7pt, top=5pt, bottom=5pt, breakable]
\centering\textit{Can a model abandon a default that violates the stated constraints and produce a correct completion?}
\end{tcolorbox}

Accordingly, our first step is to define what a default completion is. We do
so using category theory to describe how a structure given on part of a
domain continues over the rest. Specifically, when a functor defined on
part of a category is extended to the whole, the left and right Kan extensions
give two distinguished continuations characterized by their universal
properties \citep{maclane1998categories}. In our finite-set setting, both
extensions are computable from the data alone, and we call them
\emph{canonical completions}. Recent work uses Kan extensions to formulate
learning and error-minimization algorithms
\citep{shiebler2022kan,pugh2025learning}. This motivates treating the
canonical completions as candidate defaults, which we test against each
model's unconstrained answers in Section~\ref{sec:capability}.

For a model with a canonical default, we can test whether it abandons that
completion when the constraints exclude it. This is the second step, which
we call the \emph{override}. We test it on extension problems whose
constraints exclude the canonical completions and admit a single correct
structure up to renaming, and we call such a problem a \emph{jump instance}.
The exclusion distinguishes the correct answer from the default, while
uniqueness makes its structure identifiable for grading. In our certified
families, this answer contains invented elements and relations selected by
the constraints while preserving the observed data. A correct answer thus
supplies structure beyond the observed tables. We certify exclusion and
uniqueness by checking them against every bounded candidate completion with
an algorithm we release. To extend certification beyond exhaustive checking,
we further prove a family theorem that certifies arbitrarily long chains
and computes their chance levels in closed form.

To quantify the override, we evaluate a model on the same instance in
independent calls with the excluding constraints removed and with them supplied.
We compare how often the answer is canonical without the constraints, the
\emph{default-canonicity rate}, with how often it stays canonical under them,
the \emph{Kan-default rate} (Section~\ref{sec:capability}). We also grade
correctness and check matched controls that request a canonical completion.
Our evaluation covers fourteen models and 259 certified instances from three
families. Across the main and supplementary runs, the canonical completion
returns once in $13{,}300$ constrained answers. Successful overrides occur
where canonical calibration accompanies high constrained accuracy.
On the primary chain rendering, three API models return canonical answers
in $106/108$ calibration calls and correct answers in $159/162$ constrained
calls, passing all $54$ matched controls. Several open models also combine
canonical calibration with high accuracy on larger chains and on instances
with distinct Kan extensions. Across the tested configurations, the smallest
certified success probability for a uniformly drawn bounded extension is
$6.1\times10^{-10}$. These results show that several models can correctly
override a measured canonical default under the supplied constraints.

Having observed correct overrides at the second step, we formalize the third
and fourth steps without measuring them, and we hope this work paves the path
toward measuring the full jump. These later steps ask how a completion
agrees with additional evidence and how successive commitments affect its
continuation. In our current test, we provide the constraints and the form
of the answer, so the model finds the answer inside a setting we have already
fixed, which we call Tier~1. The broader question raised by Einstein's leap
also concerns how a model proposes the constraints and the answer space
itself \citep{zahavy2026jump}. We discuss this distinction in
Section~\ref{sec:scope}.

\paragraph{Organization.} The rest of the paper is organized as follows.
Section~\ref{sec:prelim} reviews the necessary knowledge in category theory.
Section~\ref{sec:formulation} defines jump instances, their matched controls,
and the calibrated capability. Section~\ref{sec:props} proves this
formulation is well-posed under the stated assumptions and supplies
computable chance levels, including a family theorem for arbitrarily long
chains. Section~\ref{sec:beyond} formalizes the two later steps, relating
internal to predictive correctness and showing how a later constraint can
exclude an earlier admissible commitment. In Section~\ref{sec:experiments},
we apply the test to fourteen models across three certified families and
identify successful overrides through canonical calibration and correct
constrained answers. Finally, Section~\ref{sec:scope} states what a positive
or a negative result of this test would and would not show. Section~\ref{sec:conclusion}
concludes the paper.

\raggedbottom
\setcounter{topnumber}{3}
\setcounter{bottomnumber}{2}
\setcounter{totalnumber}{5}
\renewcommand{\topfraction}{0.90}
\renewcommand{\bottomfraction}{0.85}
\renewcommand{\textfraction}{0.08}
\renewcommand{\floatpagefraction}{0.72}
\setlength{\textfloatsep}{10pt plus 2pt minus 2pt}
\setlength{\floatsep}{8pt plus 2pt minus 2pt}
\setlength{\intextsep}{9pt plus 2pt minus 2pt}
\setlength{\abovecaptionskip}{4pt}
\setlength{\belowcaptionskip}{0pt}
\makeatletter
\setlength{\@fptop}{0pt}
\setlength{\@fpsep}{10pt plus 2pt minus 2pt}
\setlength{\@fpbot}{0pt plus 1fil}
\makeatother

\section{Preliminaries}\label{sec:prelim}

To test whether a model abandons a default completion, we first specify what
is being completed and how two answers are compared. We use finite sets and
function tables, which make both the completion and its constraints directly
checkable. Category theory then supplies two distinguished completions of the
same partial tables.

\paragraph{The Extension Problem.}
A finite category $\cC$ specifies objects, morphisms between them, and an
associative composition table with identities. A functor $F\colon\cC\to\FinSet$
assigns a finite set to each object and a function to each morphism, preserving
identities and composition. We use $\cD=\FinSet$, the skeletal category of finite sets,
whose objects are $[n]=\{1,\ldots,n\}$ for $n\ge0$ and whose morphisms are all
functions. This convention represents each finite structure by explicit tables.

The observed tables form a functor $\Fz$ on a nonempty proper full subcategory
$\cCz\subset\cC$. Fullness means that the data specifies every morphism
between observed objects. Writing $\incl\colon\cCz\hookrightarrow\cC$ for the
inclusion, we ask for an extension that supplies the remaining sets and
functions while preserving these observations. The possible extensions form
\[
\Ext(\Fz):=\{F\colon\cC\to\cD\mid F\circ\incl=\Fz\}.
\]
The equality fixes the observed tables, so the model's choices concern the
new objects and the morphisms incident to them.

\paragraph{A Running Example.}
Consider $a\xrightarrow{f}b\xrightarrow{g}c$ with $h=g\circ f$, and observe
$a,c$ together with $\varphi:=\Fz(h)\colon X\to Y$. Completing the hidden
object $b$ means choosing a finite set $B$ and maps
$u\colon X\to B$, $v\colon B\to Y$ such that $v\circ u=\varphi$.
The data determines the composite, while the intermediate set and its two
maps remain to be supplied. Structural constraints can distinguish among
these factorizations, as the seed construction in Section~\ref{sec:props}
will demonstrate.

\paragraph{Equality up to Renaming.}
The names assigned to elements of $B$ do not change a factorization's
structure. More generally, two extensions $F,F'$ are \emph{gauge-equivalent}
if bijections $\sigma_c\colon F(c)\to F'(c)$ fix the observed objects pointwise
and commute with every function table. Thus, for each morphism $f\colon x\to y$,
\[
F'(f)=\sigma_y\circ F(f)\circ\sigma_x^{-1},
\qquad \sigma_c=\id\quad(c\in\cCz).
\]
We write $[F]$ for this equivalence class, its \emph{gauge component}.
For fixed sizes $n_c=|F(c)|$, these renamings form the group
$G_F=\prod_{c\notin\cCz}\Sym(n_c)$.
Uniqueness below means uniqueness of a gauge component, allowing a model to
choose its own names for invented elements.

\paragraph{The Two Kan Completions.}
The left and right Kan extensions distinguish two factorizations using the
data alone. In the running example they take the intermediate set to be a
copy of the input or a copy of the output:
\[
\begin{array}{c@{\qquad}ccc}
 & B & u & v\\[2pt]
\Lan & X & \id_X & \varphi\\
\Ran & Y & \varphi & \id_Y .
\end{array}
\]
For a general category, the left Kan extension carries observed elements
along every incoming morphism and identifies them according to the observed
relations. The right Kan extension collects compatible observed values along
every outgoing morphism. Appendix~\ref{app:kan} gives the finite formulas
and explains how they assign function tables to morphisms.

These constructions have complementary universal properties. After choosing
representatives that agree exactly with $\Fz$, the left extension $\Lans$ has
a unique data-fixing natural transformation to every extension $F$, and the
right extension $\Rans$ has a unique one from $F$.
Here a natural transformation is a family of maps commuting with the function
tables, with identity maps on the observed objects. These universal
properties determine the two completions up to gauge
\citep{maclane1998categories}; Proposition~\ref{prop:wellposed} supplies the
existence and representative-independence argument.

The universal properties justify using the Kan extensions as mathematical
references. Their relationship to a model's default is a separate empirical
question. Prior work constructs learning procedures using Kan extensions
\citep{shiebler2022kan} and represents error minimizers as Kan extensions in
suitably constructed categories \citep{pugh2025learning}. These results do not
identify an LLM's answer with either completion of our fixed finite problem.
We therefore compute the two references from the data and measure whether a
model produces them before testing the effect of the constraints.

\section{Formulation}\label{sec:formulation}

The two references let us construct a test in which satisfying the constraints
requires a noncanonical answer. To interpret success as an override, we also
need to establish that the model held a canonical default on the same data.
We first define the constrained problems and their controls, then combine
calibration with a success criterion.

\subsection{Instances}\label{sec:instances}

A \emph{Tier-1 instance} is a tuple
$S=(\cC,\cCz,\cD,\Fz,\cK,N)$, where $\cK$ specifies the constraints on a
completion and $N$ bounds the size of each new object. Both the constraints
and the form of the answer are supplied to the model. We now specify the conditions needed for finite certification.

In the finite-set setting of Section~\ref{sec:prelim}, the constraints are
polynomial-time decidable on the tables and invariant under renaming of
invented elements. Their language describes structural relations and size
bounds without naming answer elements. The constraints also entail the
stated bound $N$ at each new object, so every admissible answer belongs to a
finite, enumerable search space.

These requirements make structural answers checkable and their number
computable. Appendix~\ref{app:kan} records the complete assumptions, including
the effective conditions for other codomains, and
Appendix~\ref{app:discipline} describes the constraint-authoring checks.
We use the two Kan operators as the default library; the next definition also
permits comparison with a larger fixed library.

\subsection{The Canonical Class}\label{sec:canon}

We fix a library of reference operators before generating instances. Its
outputs specify which completions a jump instance will exclude.

\begin{definition}[Constraint-blind Operator and Canonical Class]\label{def:canon}
An extension operator $\Lambda$ reads $(\cC,\cCz,\Fz)$ without reading
$\cK$ or $N$. On its decidable domain it terminates with an extension whose
gauge component is independent of its internal choices; outside that domain
it declines. For a fixed library $\cL$, define
\[
\Canon_\cL(\Fz):=
\bigcup_{\substack{\Lambda\in\cL\\\Lambda\text{ defined on }(\cC,\cCz,\Fz)}}
[\Lambda(\cC,\cCz,\Fz)],
\qquad
\Kan(\Fz):=[\Lans]\cup[\Rans].
\]
\end{definition}

For the Kan library, the canonical class consists of the two Kan components.
These may coincide, as they do when the running example has singleton data.

\subsection{Jump and Control Instances}\label{sec:jump}

The constraints select the \emph{admissible set}
\[
\Adm(S):=\{F\in\Ext(\Fz)\mid P(F)\text{ for every }P\in\cK\}.
\]
To isolate an override, this set should contain a correct structure while
excluding the canonical class. We additionally require structural uniqueness,
so every correct answer represents the same completion.

\begin{definition}[Jump Instance]\label{def:jump}
An instance $S$ is a \emph{jump instance relative to $\cL$} if it satisfies:
\begin{itemize}[leftmargin=1.6em,itemsep=1pt,topsep=2pt]
\item[\textbf{(J1)}] \emph{Solvability:} $\Adm(S)\neq\emptyset$.
\item[\textbf{(J2)}] \emph{Non-canonicity:}
$\Adm(S)\cap\Canon_\cL(\Fz)=\emptyset$.
\item[\textbf{(J3)}] \emph{Identifiability:} $\Adm(S)$ is a single gauge
component.
\item[\textbf{(J4)}] \emph{Support:} every new object has a morphism from or
to an observed object. The strong form requires an incoming morphism at
every new object.
\end{itemize}
Unqualified jump instances use $\cL=\cL_{\Kan}$.
\end{definition}

Conditions (J1)--(J3) ensure that success requires the unique admissible
structure outside the library. Support connects each new object to the data;
the strong form gives the left Kan construction an incoming index at each
new object. To check (J2), we evaluate the constraints on each defined library
output.

The same extension problem can also test whether a model follows constraints
that select a canonical answer. This comparison distinguishes failure to
abandon the default from failure on the matched extension task.

\begin{definition}[Control Instance]\label{def:control}
An instance $S'$ is a control targeting a defined $\Lambda\in\cL$ if it
satisfies (J1), (J3), (J4), and
$\Adm(S')=[\Lambda(\cC,\cCz,\Fz)]$.
A matched jump--control pair shares $(\cC,\cCz,\cD,\Fz,N)$ and its answer
format, while its constraint content differs.
\end{definition}

\subsection{Chance}\label{sec:chance}

A noncanonical answer can also arise from guessing. To quantify a specific
guessing baseline, we sample uniformly from the bounded extensions
\[
\ExtN(\Fz):=\{F\in\Ext(\Fz)\mid s(F(c))\le N
\text{ at every new object}\}.
\]
Assumptions~\ref{ass:C}--\ref{ass:D} make this set finite and enumerable,
and Assumption~\ref{ass:N} places every admissible completion inside it.
For instances with a nonempty bounded answer space, the uniform success
probability is
\begin{equation}\label{eq:chance}
\chance(S):=\frac{|\Adm(S)|}{|\ExtN(\Fz)|}.
\end{equation}
Both counts concern labeled function tables. Uniform sampling over gauge
components would define a different baseline, because components can have
different numbers of labeled representatives.

This baseline respects the data, the functor equations, and the declared
bound while ignoring the remaining constraints. It does not bound the success
of every constraint-blind strategy, which may assign a different distribution
to the same tables. The declared bound also affects the denominator, so we
fix and report it with each generator; Appendix~\ref{app:exp} examines this
sensitivity.

\subsection{Jump Capability: Calibration, Kan-default Rate, Margin}
\label{sec:capability}

The chance level measures guessing, while calibration determines whether the
canonical class describes the model's default. We present each instance with
its constraint block replaced by fixed neutral filler and call the resulting
input the \emph{ablated twin} $S^\circ$. It retains the data, answer format,
and declared bound. The learner's answer to this input supplies the
unconstrained default for comparison with its answer to $S$.

Let $\mathcal G_n$ generate certified jump instances at scale $n$.
For each sampled instance, calibration and constrained responses are drawn
independently under fixed sampling settings. Probabilities below include both
instance sampling and learner randomness. We record
\begin{equation}\label{eq:kd}
\begin{aligned}
\DC_L(n)&:=\Pp[L(S^\circ)\in\Canon_\cL(\Fz)],\\
\KD_L(n)&:=\Pp[L(S)\in\Canon_\cL(\Fz)],
\qquad \Delta_L(n):=\DC_L(n)-\KD_L(n).
\end{aligned}
\end{equation}
For the calibrated test, the reference outputs are required to satisfy the
declared answer bound, so a canonical default is a valid completion of the
ablated task. The first quantity is the \emph{default-canonicity rate}; the second is the
\emph{canonical-default rate under constraints}, or \emph{Kan-default rate}
for $\cL_{\Kan}$. Their difference is the \emph{override gap}.

A decrease in canonical answers establishes a change in response
probabilities, but a changed answer may still violate the constraints.
We therefore score constrained success by membership in $\Adm(S)$, with
outputs outside $\ExtN(\Fz)$ scoring zero and recorded separately.
The capability criterion combines this success measure with calibration and
the matched controls.

\begin{definition}[Calibrated Jump Capability]\label{def:capability}
Fix thresholds $\varepsilon>0$, $\kappa,\delta\in[0,1)$, a generator
$\mathcal G_n$ of certified $\cL$-jump instances, and a matched control
generator $\mathcal G'_n$, all before evaluation. The learner $L$ jumps at
scale $n$ with margin $\varepsilon$, relative to $(\cL,\kappa,\delta)$, if:
\begin{itemize}[leftmargin=1.8em,itemsep=1pt,topsep=2pt]
\item[\textbf{(C0)}] \emph{Calibration validity:}
$\DC_L(n)\ge1-\kappa$.
\item[\textbf{(C1)}] \emph{Success above the uniform baseline on calibrated
instances:}
\[
\Pp[L(S)\in\Adm(S)\mid L(S^\circ)\in\Canon_\cL(\Fz)]
\ge
\E[\chance(S)\mid L(S^\circ)\in\Canon_\cL(\Fz)]+\varepsilon.
\]
\item[\textbf{(C2)}] \emph{Matched control:}
$\Pp_{S'\sim\mathcal G'_n}[L(S')\in\Adm(S')]\ge1-\delta$.
\end{itemize}
\end{definition}

Condition (C0) checks whether this library is an appropriate reference for
this learner at this scale. When it fails, we report the observed default
distribution without attributing constrained behavior to an override of that
library. Conditional on a canonical calibration response, (C1) requires a
correct answer above the uniform baseline. Condition (C2) checks success when
the constraints instead select the canonical completion.

\begin{proposition}[Null Learners]\label{prop:null}
Fix $\Lambda\in\cL$ defined on every tested input, and let $L_\Lambda$ return
its output independently of the constraints. On matched controls targeting
$\Lambda$, this learner succeeds with probability one; on jump instances it
fails with probability one. It has $\DC=\KD=1$ and $\Delta=0$, so it fails
(C1) for every $\varepsilon>0$. A learner sampling uniformly from $\ExtN$
on each independent call attains the conditional chance term in (C1).
\end{proposition}

This result identifies the behavior the test is designed to distinguish
(proof in Appendix~\ref{app:proofs}). High calibration and control success,
together with a high constrained default rate, indicate persistence of the
excluded reference. A low Kan-default rate indicates departure from that
reference, and constrained accuracy determines whether the departure succeeds.
The remaining requirement is to construct instances for which these outcomes
can be certified.

\section{Constructing and Certifying Jump Instances}\label{sec:props}

The test requires a certificate that the admissible answer exists, is unique
up to renaming, and lies outside the canonical class. We first construct a
small instance where these properties can be seen directly. Its structure
then extends to a family for which both the certificate and the uniform
chance level can be obtained without enumerating the answer space.

\subsection{A Certified Seed}\label{sec:seed}

Return to the factorization $a\to b\to c$ from Section~\ref{sec:prelim},
now with singleton observations at both ends. The two Kan extensions assign
a singleton to $b$. To force a different completion, we add a transformation
$t\colon b\to b$ that preserves the incoming and outgoing maps and satisfies
$t^2=\id_b$. We then require its completed action to be nontrivial while
bounding the hidden set's size.

\begin{proposition}[Existence of a Certified Seed]\label{prop:seed}
Let $S^*=(\cC,\cCz,\FinSet,\Fz,\cK,4)$, where $\cC$ has objects $a,b,c$,
identities, and morphisms $f\colon a\to b$, $g\colon b\to c$,
$h=g\circ f$, $t\colon b\to b$, with
\[
t\circ f=f,\qquad g\circ t=g,\qquad t^2=\id_b.
\]
The data on the full subcategory $\{a,c\}$ assigns a singleton to each object,
and the constraints are
\[
\mathrm{K1}:F(t)\neq\id_{F(b)},\qquad
\mathrm{K2}:|F(b)|\le3.
\]
Then $S^*$ is a jump instance. Both Kan extensions are the trivial singleton
completion and violate K1. Every admissible completion has three elements
at $b$, with $F(f)$ selecting a fixed point and $F(t)$ exchanging the other
two. These completions form one gauge component of three labeled tables,
while $|\Ext_4(\Fz)|=25$, giving $\chance(S^*)=3/25$.
\end{proposition}

The relation $t\circ f=f$ forces an observed fixed point. A nonidentity
involution also requires two exchanged elements, and the size constraint
leaves room for precisely these three points. Thus the answer adds two
elements outside the incoming image, carrying a nontrivial action absent
from the data and the Kan completions. 

Replacing K1 and K2 by $|F(b)|\le1$ gives the matched control, with the same
declared bound $N=4$. Its admissible answer is the singleton Kan completion.
The jump and control therefore require different answers to the same observed
tables. Appendix~\ref{app:proofs} proves the exact counts, and
Appendix~\ref{app:seed} presents the seed as a benchmark prompt.

\subsection{Finite Certification}\label{sec:certification}

For the seed, the forcing argument determines the admissible structure.
For a general finite instance, we can certify the same properties by checking
all bounded completions. The declared bound makes this enumeration finite;
gauge invariance then allows admissible tables to be grouped by their
structure.

Three facts justify this procedure. The Kan extensions have strict
representatives whose gauge components are independent of construction
choices (Proposition~\ref{prop:wellposed}). Admissibility and canonicality are
preserved by renaming (Proposition~\ref{prop:gauge}). Each gauge component is
a finite orbit, with a common size profile
(Proposition~\ref{prop:finite}). Appendix~\ref{app:proofs} states and proves
these supporting results. Together with the effective finite answer space,
they make (J1)--(J4) and Equation~\eqref{eq:chance} computable for the
certified instances.

The enumeration cost grows with the bounded answer space, even when its
admissible structure has a short description. To retain certification at
larger scales, we next derive that structure and its count for an entire
family.

\subsection{The Family Theorem: Certification without Enumeration}
\label{sec:family}

For $m\ge1$ and primes $p_1,\ldots,p_m$, connect the objects in a chain
$a\to b_1\to\cdots\to b_m\to c$, with singleton data at $a,c$.
At each hidden object $b_i$, add an endomorphism $t_i$ of order $p_i$ that
absorbs into the adjacent maps. In a completion, this requires $F(t_i)$ to
fix the incoming image and the outgoing map to be constant on its orbits.
The constraints at that object are
\[
\mathrm{K1}_i:F(t_i)\neq\id,
\qquad \mathrm{K2}_i:|F(b_i)|\le1+p_i.
\]
Write $S(m,\vec p\,)$ for this \emph{pointed-chain instance}, using
$N=\max_i(1+p_i)+1$ unless another bound is specified. The seed is
$S(1,(2))$. Appendix~\ref{app:family} gives the finite composition tables
underlying this presentation.

\begin{theorem}[Family Theorem]\label{thm:family}
For every $m\ge1$, primes $p_1,\ldots,p_m$, and
$N\ge\max_i(1+p_i)$, the pointed-chain instance is a jump instance with the
strong form of (J4).
\textup{(i)} Its admissible completions have one fixed point and one
$p_i$-cycle at each hidden object, with incoming and outgoing maps constant.
\textup{(ii)} They form a single gauge component of cardinality
\[
|\Adm(S)|=\prod_{i=1}^m(1+p_i)(p_i-1)!.
\]
\textup{(iii)} Both Kan extensions are the all-singleton completion and violate every
$\mathrm{K1}_i$. \textup{(iv)} The bounded count $|\ExtN(\Fz)|$ is computable by a transfer
recursion using $O(mN^4)$ arithmetic operations. For $p_i=2$ and $N=4$,
\begin{equation}\label{eq:family-chance}
\chance(S(m,(2,\ldots,2)))\le(3/25)^m.
\end{equation}
If an order $p_i\ge2$ is composite, the corresponding construction fails
identifiability.
\end{theorem}

Primality makes each nontrivial cycle have length $p_i$. The incoming image
requires a fixed point, so the size bound forces one fixed point and one
cycle, generalizing the seed's argument. Renaming acts transitively on these
choices, which gives structural uniqueness. To count all bounded extensions,
the recursion instead tracks the number of fixed points and cycles at each
object, together with the maps connecting consecutive objects.
Appendix~\ref{app:family} derives the recursion and proves
Equation~\eqref{eq:family-chance}.

The theorem allows instance size to grow while the uniform chance level
decreases at least exponentially along this subfamily. Certification then
requires checking that an instance belongs to the family and evaluating the
counting formulas. This provides exact grading at each scale without enumerating the bounded
extensions.

\begin{algorithm}[H]
\caption{Certification of a Kan Jump Instance}\label{alg:certify}
\begin{algorithmic}[1]
\Require finite instance $S$ satisfying A1--A5, with verified bound entailment
\State compute strict Kan representatives $\Lans,\Rans$ using
Appendix~\ref{app:kan}
\If{$S$ matches the pointed-chain construction with prime orders}
  \State obtain the admissible orbit and both counts from
Theorem~\ref{thm:family} and Appendix~\ref{app:family}
\Else
  \State enumerate $\ExtN(\Fz)$, test $\cK$, and group admissible tables
into gauge orbits
\EndIf
\State check (J1), exclusion of both Kan representatives (J2), one
admissible orbit (J3), and support (J4)
\State \Return the certificate and $|\Adm(S)|/|\ExtN(\Fz)|$ if all checks pass
\end{algorithmic}
\end{algorithm}

Algorithm~\ref{alg:certify} combines the two certification routes. Each route certifies success against a unique admissible structure. The next section examines what changes when the current
constraints leave several structures possible.

\section{Correctness and Chained Jumps}\label{sec:beyond}

Identifiability makes the override test unambiguous, because every correct
answer represents the same structure. It also removes the choice among
several structures that fit the current constraints. To study that choice,
we now allow more than one admissible component and ask how a completion is
judged against a hidden world or against later constraints.

\subsection{Internal versus Predictive Correctness}

A completion is \emph{internally correct} when it belongs to $\Adm(S)$.
This tests agreement with the supplied constraints. To compare it with the
underlying world, we introduce additional objects and a ground-truth functor
that are withheld when the model answers.

\begin{definition}[Extended Instance]\label{def:extended}
An extended instance is $T=(S,\cC',G)$, where $S$ satisfies A1--A5,
$\cC'$ is a finite category containing $\cC$ as a full subcategory, and
$G\colon\cC'\to\cD$ agrees with $\Fz$ on $\cCz$. The learner receives $S$
while $(\cC',G)$ is withheld. We call $T$ \emph{sound} when
$G|_{\cC}\in\Adm(S)$.
\end{definition}

Soundness requires the supplied constraints to hold in the ground-truth
world. Under this assumption, a remaining disagreement concerns which
admissible structure is correct. We use a structural criterion for this
comparison.

\begin{definition}[Predictive Correctness]\label{def:predictive}
A completion $F\in\Ext(\Fz)$ is predictively correct for $T$ if it extends
to $\cC'$ as a functor isomorphic to $G$, through an isomorphism that is the
identity on $\cCz$.
\end{definition}

This criterion asks for agreement with the ground-truth structure up to
renaming. Agreement with a finite collection of later observations is weaker,
since distinct structures may reproduce those observations.

\begin{proposition}[Identifiability and Predictive Correctness]\label{prop:predchar}
For any extended instance, the predictively correct completions are exactly
$[G|_{\cC}]$. For a sound extended instance, this component lies in
$\Adm(S)$, and every internally correct completion is predictively correct
if and only if (J3) holds.
\end{proposition}

The proof transports $G$ along a renaming of its restriction
(Appendix~\ref{app:beyond}). Under soundness, the unique admissible component
already contains the true completion, so the certified override instances
leave no further structural choice. To separate the two correctness criteria,
we replace (J3) by the condition that $\Adm(S)$ contains exactly $k$ gauge
components, calling the instance \emph{$k$-ambiguous}.

\begin{proposition}[Separation by Held-out Observations]\label{prop:separation}
Relax the seed's constraint $|F(b)|\le3$ to $|F(b)|\le4$, retaining $N=4$.
The resulting instance $S_{\mathrm{sep}}$ has two admissible components:
$Y_3$, the three-element seed structure, and $Y_4$, the same structure with
one additional fixed point outside the incoming image. There is a sound
extended instance with $G|_{\cC}\in Y_4$ whose held-out probe observations
are compatible with every member of $Y_4$ and with no member of $Y_3$.
Thus both components are internally correct and exactly $Y_4$ is
predictively correct.
\end{proposition}

The extra fixed point explains the separation. A new incoming probe can
select it, and a new outgoing probe can distinguish it from the fixed point
reached by the original data. The three-element completion has a single
fixed point, so it cannot reproduce these two different probe values.
Appendix~\ref{app:beyond} constructs the probes and proves the refutation.
The example shows how later observations can resolve ambiguity that the
current constraints permit.

\subsection{Chains, Monotonicity, and Entrenchment}

Later observations can also be tested after a model has committed to a
completion. To represent such a commitment, we promote the completed tables
to observed data and extend them again. This makes the effect of retaining
an earlier answer mathematically explicit.

\begin{definition}[Incorporation and Chain]\label{def:chain}
A successor scheme for $S$ consists of a finite category $\cC'$ containing
$\cC$ as a proper full subcategory, a constraint list $\cK'$, and a bound
$N'$. The constraints satisfy A4, entail the bound on the new objects, and
are invariant under renaming at every object, including objects in $\cC$.
For a commitment $X\in\Adm(S)$, define
\[
\iota(S,X,T):=(\cC',\cC,\cD,X,\cK',N'),
\qquad T=(\cC',\cK',N').
\]
A chain iterates incorporation along a fixed sequence of successor schemes,
choosing an admissible commitment at each completed stage.
\end{definition}

The stronger invariance condition lets the successor use the same constraint
text for different renamings of the preceding answer. Equations, size bounds,
and marked limit conditions expressed in the structural language have this
property. It ensures that the future verdict depends on the commitment's
structure.

\begin{theorem}[Incorporation Is Well-defined]\label{thm:chainwd}
The incorporated instance satisfies A1--A5. Gauge-equivalent commitments
yield data-isomorphic successors with corresponding admissible components
of equal size and equal bounded extension counts. Their chance levels agree
whenever the bounded answer spaces are nonempty. For $\cL_{\Kan}$,
(J1)--(J4) have equal truth values on these successors.
\end{theorem}

Appendix~\ref{app:beyond} proves the result by transporting the completed
function tables. Its canonicality conclusion also holds for libraries whose
operators respect isomorphisms of the data. A successor still requires its
own solvability and non-canonicity checks; incorporation by itself supplies
an extension problem, rather than a further certified jump.

Every successor extends the preceding commitment exactly. The earlier
admissibility conditions therefore remain binding, as
Lemma~\ref{lem:mono} shows after transporting them to a common category.
To compare alternative commitments, we keep the future schemes fixed and
ask which earlier answers admit some continuation.

\begin{definition}[Retrospective and Ultimate Correctness]\label{def:retro}
Fix the successor schemes and the commitments preceding stage $k$.
An answer $X\in\Adm(S_k)$ is retrospectively correct through stage
$\ell\ge k$ if there exists $H\colon\cC_\ell\to\cD$ with
$H|_{\cC_k}=X$ satisfying every constraint imposed at stages $k+1,\ldots,\ell$.
Let $\mathcal R_\ell(k)$ be the set of these answers. Ultimate correctness
means membership in $\bigcap_{\ell\ge k}\mathcal R_\ell(k)$.
\end{definition}

The sets $\mathcal R_\ell(k)$ decrease with $\ell$ and are unions of gauge
components. Since $\Adm(S_k)$ is finite, the sequence eventually stabilizes,
although this gives no bound on the stage of stabilization. Under (J3),
all admissible answers have the same continuation verdict. Ambiguity is
therefore needed for later constraints to eliminate one admissible component
while preserving another.

\begin{theorem}[Failure of Continuation]\label{thm:entrench}
There is a two-stage scheme starting from the relaxed seed with
$\Adm(S_1)=Y_3\sqcup Y_4$. Stage two adds a probe object $d$ and a map
$r\colon d\to b$ that represents the fixed-point set of the committed
involution, together with $|F(d)|\le1$. Every commitment in $Y_3$ has one
admissible continuation, which is its right Kan extension and hence a
control completion. Every commitment in $Y_4$ has no admissible continuation.
Consequently $\mathcal R_2(1)=Y_3$.
\end{theorem}

The probe requires one element for each fixed point. A commitment in $Y_3$
has one such point and satisfies the probe bound; a commitment in $Y_4$ has
two and violates it. Appendix~\ref{app:beyond} expresses this requirement
as an equalizer and proves the two verdicts.
The result isolates an obstruction caused by retaining an earlier admissible
answer. Measuring whether a model revises that answer requires allowing it
to reopen the commitment, which strict incorporation fixes as data.

The later two steps thus require ambiguity and additional evidence beyond
the identifiable override instances. Section~\ref{sec:experiments} evaluates
the calibrated override on certified problems; predictive selection and
revision across stages remain separate measurements.

\section{Experiments}\label{sec:experiments}

Having certified the instances, we now test whether models replace their measured
default with a correct completion when the constraints exclude that default. We first
check calibration and controls, then examine the Kan-default rate together with jump
accuracy across three certified families.

\paragraph{Models and Instances.} We evaluate fourteen models. Four are reached through a
commercial API, GPT-5.6 Luna Pro, Claude Sonnet 5, Gemini 3.1 Pro, and DeepSeek V4 Pro,
and ten are open-weight models served locally, for which we also score candidate answers
using token log probabilities. Eight of the open models write a reasoning trace before
their answer, whereas Llama 3.3 70B and Qwen2.5 72B answer directly. The instances come
from three certified families. The pointed-chain family supplies the nine instances of
Section~\ref{sec:props}, six certified by enumeration and three by
Theorem~\ref{thm:family}, with chance levels from $0.13$ down to $8.1\times10^{-7}$.
The scaled family uses eighteen chain configurations with ten nonce vocabularies each,
producing 180 instances and carrying chance down to $6.1\times10^{-10}$. An asymmetric
family supplies 70 instances over seven configurations with different numbers of observed
input and output slots. There $\Lan$ and $\Ran$ differ, the constraints exclude both,
and each configuration is certified by enumeration.

\paragraph{Protocol.} Each instance is rendered with nonce vocabulary inside a neutral
cover story. We obtain independent calibration and jump answers by removing or supplying
the excluding constraints. A matched control selects the canonical completion
through a different requirement block.
The asymmetric family has one control for each Kan extension. The main pointed-chain
runs use six jump, four calibration, and two control calls per instance and rendering,
the first greedy and the rest at temperature $0.7$. Open models receive
both cover stories; API models receive the primary story plus eight jump-only wording
samples each. The larger families use the schedules in Appendix~\ref{app:exp}.
Answers are graded by the certified structural characterizations, so membership is exact
and gauge-invariant. Truncated answers count as failures and remain in the denominators.
The main output budget is $12{,}000$ tokens; a separate arm uses $18{,}000$ tokens.
Together, the three families and the additional wording, scratchpad, and budget runs yield
$30{,}768$ graded answers, of which $13{,}300$ carry the excluding constraints.
Appendix~\ref{app:exp} gives the sampling settings and the counts for each run.

\paragraph{Instrument Validation.} On the pointed-chain family, the API models return a
Kan extension in $129$ of $144$ calibration calls, and the eight open models with a
reasoning trace in $297$ of $576$. Table~\ref{tab:main} reports these rates by model,
including unfinished answers. Calibration is weak for Qwen3.6 35B-A3B and Qwen3.6 27B,
which return canonical answers in $7/72$ and $0/72$ calls and frequently exhaust their
budget. Llama 3.3 70B and Qwen2.5 72B finish their calibration calls but return a Kan
extension in only $8/144$ of them, placing most of their outputs outside the library on
this family. Controls are passed in $69/72$ API calls, $227/288$ calls from the reasoning
models, and all $72$ calls from the two directly answering models. Where calibration and
controls are reliable, constrained accuracy determines whether replacing the default
succeeds. Appendix~\ref{app:exp} also reports the conditional success quantity in (C1),
matching independent calls by instance, rendering, and temperature. The two cover stories
yield similar aggregate jump counts for the reasoning models, $242/432$ and $254/432$.
The pooled chain calibration and canonical-rate intervals appear in
Figure~\ref{fig:instrument} in Appendix~\ref{app:figs}.

\begin{table}[tbp]
\centering
\caption{Original-chain calibration ($\DC$) and control counts, with jump accuracy for
each family. All calls remain in the denominators. $\KD$ counts canonical answers
across these families, excluding scratchpad and extra-budget runs. Daggers mark
models with over $40\%$ length-limited jump calls; dashes mark unevaluated
families. Appendix~\ref{app:exp} gives the sampling scopes.}
\label{tab:main}
\small
\definecolor{LightGray}{gray}{0.92}
\rowcolors{2}{white}{LightGray}
\renewcommand{\arraystretch}{1.12}
\begin{tabular}{lcccccc}
\toprule
\textbf{model} & $\DC$ & \textbf{control} & \textbf{chain} & \textbf{scaled} & \textbf{asym.} & $\KD$\\
\midrule
\multicolumn{7}{l}{\emph{Frontier models behind an API}}\\
GPT-5.6 Luna Pro & 35/36 & 18/18 & 0.95 & -- & -- & 0/62\\
Claude Sonnet 5 & 36/36 & 18/18 & 0.98 & -- & -- & 0/62\\
Gemini 3.1 Pro & 35/36 & 18/18 & 1.00 & -- & -- & 0/62\\
DeepSeek V4 Pro & 23/36 & 15/18 & 0.71 & -- & -- & 0/62\\
\multicolumn{7}{l}{\emph{Open-weight models with a reasoning trace}}\\
Qwen3.6 35B-A3B$^{\dagger}$ & 7/72 & 3/36 & 0.08 & 0.02 & 0.22 & 0/1208\\
Qwen3 32B & 55/72 & 36/36 & 0.97 & 0.89 & 1.00 & 0/1208\\
Qwen3.6 27B$^{\dagger}$ & 0/72 & 17/36 & 0.19 & 0.08 & 0.63 & 0/1208\\
gpt-oss 20B & 56/72 & 35/36 & 0.97 & 0.93 & 0.99 & 0/1208\\
Qwen3 14B & 66/72 & 36/36 & 0.93 & 0.86 & 1.00 & 0/1208\\
Qwen3.5 9B$^{\dagger}$ & 34/72 & 32/36 & 0.41 & 0.26 & 0.90 & 0/1208\\
Qwen3 8B & 60/72 & 36/36 & 0.77 & 0.62 & 1.00 & 0/1208\\
Qwen3.5 4B$^{\dagger}$ & 19/72 & 32/36 & 0.28 & 0.19 & 0.88 & 0/1208\\
\multicolumn{7}{l}{\emph{Open-weight models answering directly}}\\
Qwen2.5 72B & 2/72 & 36/36 & 0.00 & 0.00 & 0.01 & 0/1208\\
Llama 3.3 70B & 6/72 & 36/36 & 0.01 & 0.03 & 0.01 & 1/1208\\
\bottomrule
\end{tabular}
\end{table}

\paragraph{Main Results.} Table~\ref{tab:main} places
canonical calibration beside performance under the excluding constraints. We first ask
how often the default remains and whether its replacement is correct, then examine
larger instances, candidate likelihoods, distinct Kan extensions, and evaluation conditions.

\textbf{Insight 1: the canonical completion returns once in $13{,}300$ constrained answers.}
Across the main and supplementary runs, one of $13{,}300$ constrained answers is canonical,
giving a pooled rate of $7.5\times10^{-5}$ and a descriptive Wilson $95\%$ upper bound of
$4.3\times10^{-4}$. The single event is a temperature-$0.7$ answer from Llama 3.3 70B on
the asymmetric family, where $553$ of its $560$ jump answers violate a rule. The canonical
completion appears in $0/3{,}586$ greedy answers and $1/9{,}714$ sampled answers; the
scratchpad arm contributes $0/540$. These counts show that the excluded default is rarely
the final answer. The calibrated cases with high accuracy establish successful replacement.
On the primary chain rendering, GPT-5.6 Luna Pro, Claude Sonnet 5, and Gemini
3.1 Pro give canonical calibration answers in $106/108$ calls, correct jump answers in
$159/162$, and pass all $54$ controls. For these models, the changed constraints reliably
select the correct non-canonical completion.

\begin{figure}[tbp]
\centering
\includegraphics[width=\textwidth]{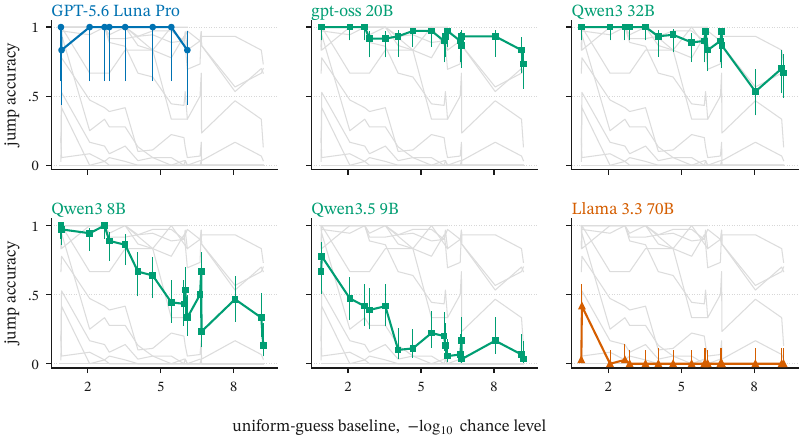}
\caption{Primary-story jump accuracy for six models spanning high, declining, and low
accuracy. Gray curves show the other models; bars are nominal Wilson $95\%$
intervals. Original and scaled calls are pooled for shared configurations.
Figure~\ref{fig:difficulty-full} gives all fourteen panels and the sample counts.}
\label{fig:difficulty}
\end{figure}

\textbf{Insight 2: larger instances separate the models in completion accuracy.}
Figure~\ref{fig:difficulty} plots accuracy across configurations using their certified
chance levels as the horizontal axis. On the scaled family, gpt-oss 20B and Qwen3 32B
achieve $0.93$ and $0.89$ accuracy over $540$ jump calls each, with canonical calibration
in $292/360$ and $264/360$ calls and all $180$ greedy controls passed by each model.
Thus correct replacement extends to larger chains, although accuracy varies across
configurations. The two groups of open models also fail differently. Across their chain,
scaled, and asymmetric jump calls, Llama 3.3 70B and Qwen2.5 72B violate a rule in
$2{,}385/2{,}416$ answers. The reasoning models violate a rule in $81/9{,}664$ and produce
$3{,}313$ format or termination failures, including $3{,}291$ length-limited responses.
These errors distinguish failure to satisfy the task from returning the default, and
make output budget a relevant control.

\textbf{Insight 3: canonical answers also receive higher teacher-forced scores.}
For each open model, we score one canonical answer and one admissible answer under the
calibration prompt, disabling thinking where the chat template supports it. The canonical
text has the higher summed token log probability in $174$ of $180$ model, instance, and
rendering combinations, with a median advantage of $14.8$ nats
(Appendix~\ref{app:figs}, Figure~\ref{fig:logprob}). This comparison favors the canonical text under immediate
answering, but the text is also shorter. Using mean token log probability instead gives
$132/180$ canonical preferences. These scores compare the two serialized answers;
sampled calibration measures how often the model actually returns a canonical completion.

\begin{figure}[tbp]
\centering
\includegraphics[width=\textwidth]{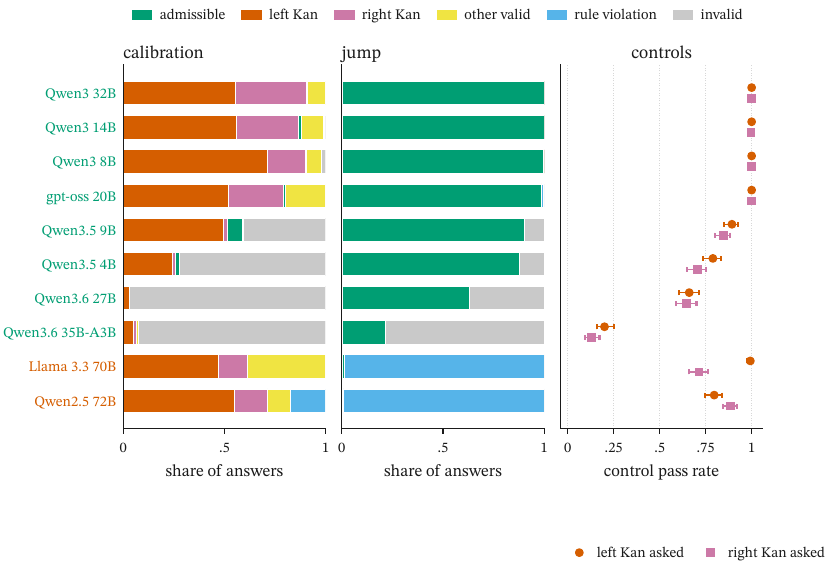}
\caption{The asymmetric family, with distinct $\Lan$ and $\Ran$ completions and both
cover stories pooled. Left: calibration outcomes over $420$ calls per model.
Middle: constrained outcomes over $560$ calls. Right: controls requesting each Kan
extension, with $280$ calls per control and nominal Wilson $95\%$ intervals.
The gray invalid category includes format and termination failures.}
\label{fig:asym}
\end{figure}

\textbf{Insight 4: when the Kan extensions differ, calibration favors $\Lan$.}
On the pointed-chain family, $\Lan$ and $\Ran$ coincide, so a canonical answer cannot
distinguish the two completions. On the asymmetric family, $1{,}755$ of $4{,}200$
calibration calls return $\Lan$ and $617$ return $\Ran$. The equal-size configuration
has $228$ and $30$ such answers out of $600$ calls, so the imbalance is present even
when the two canonical answers have the same size. Teacher forcing also favors $\Lan$
in $1{,}067$ of $1{,}400$ fixed-string comparisons. Under the excluding constraints, the eight reasoning
models reach the admissible completion in $3{,}702/4{,}480$ calls; they pass
$1{,}770/2{,}240$ controls selecting $\Ran$ and $1{,}832/2{,}240$ selecting $\Lan$.
Figure~\ref{fig:asym} also shows why calibration should be measured on each family.
Llama 3.3 70B and Qwen2.5 72B return canonical calibration answers in $258/420$ and
$299/420$ calls here, yet solve $6/560$ and $4/560$ jump calls. Canonical calibration
on this family thus coexists with frequent failure under the excluding constraints.

\textbf{Insight 5: output budget substantially changes the measured accuracy.}
We repeat four models on the pointed-chain family with the budget raised to $18{,}000$
tokens (Appendix~\ref{app:exp}, Table~\ref{tab:budget}). Correct answers rise from $177/432$ to $353/432$,
while length-limited answers fall from $250/432$ to $64/432$. These independent reruns
show substantial budget sensitivity alongside sampling variation. Calibration remains
weak for Qwen3.6 27B, which improves from $17/36$ to $27/36$ on controls but returns
canonical calibration answers in $4/72$ calls at the larger budget.
Both budgets leave the Kan-default rate at zero, while the number of correct answers changes.

\textbf{Insight 6: deployment effects are real and call for controls.}
Provider filters can prevent evaluation before a mathematical answer is produced. The
provider-screening notes record that one vendor's strictest tier refused every tested
rendering, including a pure-mathematics statement, whereas its middle tier answered them.
Three open checkpoints were also unavailable because the serving stack failed to build
their kernels. These access failures are recorded separately from task performance
(Appendix~\ref{app:discipline}).

\paragraph{Exact-search Reference.} The enumerate-and-check certifier generates bounded
extensions and retains those satisfying $\cK$. For the six enumerated chain instances,
the recorded generation and certification runs take at most $2.1$ seconds. At $m=3$,
Theorem~\ref{thm:family} counts between $1.3\times10^{6}$ and $2.4\times10^{8}$ candidate
extensions for the original configurations without enumerating them; its witness also
gives a direct structural solution. A small chance level therefore measures uniform
bounded-table guessing, rather than computational hardness. The experimental question is
whether a calibrated model supplies a correct completion when the constraints invalidate
its default (Section~\ref{sec:scope}).

\section{Discussion and Limitations}\label{sec:scope}

\paragraph{What a Pass Means.} Our test measures a component of the jump
discussed by \citet{zahavy2026jump}. The learner replaces its
\emph{measured} canonical default with the non-canonical extension selected
by the supplied constraints. Passing demonstrates constraint-guided
selection within a fixed answer space. Its interpretation requires
canonical calibration, correct constrained answers, and matched controls;
a low Kan-default rate alone also includes malformed and incorrect answers.
A symbolic solver passing Tier~1 is unproblematic, because the diagnostic
question concerns the learner's measured default and its ability to choose
another correct completion. Tier~1 yields an \emph{instrument} relative
to that learner and the chosen canonical library. The propositions of
Section~\ref{sec:beyond} supply the additional structure needed to measure
the third and fourth steps.

\paragraph{What Is Bracketed; Tier 2.} Four parts of the full jump remain
outside the instrument. The constraints $\cK$ are supplied rather than
generated from an anomaly, and the embodiment thesis
\citep{magnani2009abductive,zahavy2026jump} concerns that generation step.
The codomain $\cD$ is also supplied. Finite structured objects such as
groups and posets can be represented using sets, functions, and additional
conditions, so $\FinSet$ permits a range of structures within the given
representation. Inventing a new codomain remains a Tier~2 question.
The third part concerns evaluating a jump when its correctness is not yet
known. Section~\ref{sec:beyond} represents this by withholding the world
that would settle the answer until a later probe. The fourth concerns
what triggers the search and what the searcher intends. Tier~2 addresses
these broader questions of generation and interpretation, with Tier~1
providing a test of selection under supplied constraints.

\paragraph{Limitations.} Two limitations affect the measurement. First,
the implemented families admit direct structural solutions. The family
theorem scales certification beyond brute-force enumeration, but its small
chance levels describe uniform guessing over bounded tables rather than
computational difficulty for an algorithm using the structure. Scaling
these families therefore tests a specific range of completion problems.
Second, calibration may place much of a learner's default distribution
outside $\cL$, or unfinished responses may make that distribution difficult
to measure at the chosen budget. Llama~3.3~70B and Qwen2.5~72B rarely
return canonical calibration answers on the pointed chains, while their
canonical rates are substantially higher on the asymmetric family.
Qwen3.6~27B and 35B-A3B instead have many unfinished calibration responses.
These observations call for model-, family-, and budget-specific readings
of calibration and success.

\paragraph{Library Relativity and What the Default Is.} Non-canonicity is
defined relative to $\cL$, and our instantiation takes
$\cL=\{\Lan,\Ran\}$. Matched controls also test a broader alternative.
A fixed completion of the shared data, chosen independently of the
constraint block, fails at least one of a jump and a control that select
different answers (Proposition~\ref{prop:null}). The asymmetric family
compares the two canonical references separately (Figure~\ref{fig:asym}).
These comparisons concern returned structures; identifying the procedure
that computes them requires additional evidence.

\paragraph{Next Measurements.} The first extension is the staged chain
experiment motivated by Section~\ref{sec:beyond}. Its theorem identifies
an earlier admissible commitment that has no continuation under a later
probe. Measuring a model's response requires allowing it to revise that
commitment and observing whether the revision succeeds. The current
independent prompt arms do not measure this behavior. A second extension
would diversify the certified generators, including constraints that mark
(co)limiting cones and deliberately $k$-ambiguous instances. Within the
present test, matching answer lengths and comparing direct with indirect
exclusion would further test sensitivity to how the constraints are
expressed.

\FloatBarrier
\section{Conclusion}\label{sec:conclusion}

We have developed a four-step account of the jump and a certified test of
its override step. The certificates identify a correct completion distinct
from the Kan extensions, and calibration tests whether those
extensions describe a model's own default. Several models produce the
correct completion under the supplied constraints, providing evidence
that they can perform this part of the jump. The remaining steps concern
how such a completion agrees with further observations and how earlier
commitments affect later choices. Our theoretical results make these
questions precise, while their empirical measurement and the generation
of constraints and answer spaces remain open. We hope this framework
helps connect controlled measurements with the broader question of how
models develop and revise explanatory structures.



\bibliography{references}
\bibliographystyle{plainnat}

\appendix

\section{Related Works}

\paragraph{Jumps in LLM.} The initial work in \citet{zahavy2026jump} argues that LLMs lack the abductive move from given evidence
(denoted as $E$) to a system of axioms (denoted as $A$). That work attributes the gap to the absence
of embodied simulation and therefore proposes interactive world models as the
remedy. The argument remains programmatic rather than formal.
In parallel, \citet{farmer2026abduction} challenges the proposed requirement for
embodied grounding in abduction. Consistent with this view, systems such as AlphaEvolve
have produced novel mathematical constructions within fixed formal languages
\citep{novikov2025alphaevolve}. The direct responses also decompose the jump
differently. \citet{gangloff2026assessing} argues that the jump is not a single capacity and
asks benchmarks to disaggregate by mechanism. \citet{saldana2026stratified}
decomposes discovery into operators on a stratified graph and conjectures that
the irreducible residue is the retraction of axioms. \citet{murphy2026model}
engineers around the limitation by letting a Bayesian loop request new
hypotheses from the model.
Our account can coexist with each of these decompositions, since the
Kan-default rate presupposes only a computable default and a certified
exclusion rather than a full theory of the jump. Under our formulation, the disputed thesis corresponds to a high Kan-default
rate that persists even when the in-context constraints are certified to
falsify the default. The thesis in this rendering is directly falsifiable.

\paragraph{Category Theory and Learning.} \citet{shiebler2022kan} casts
generalization from partial data as Kan extension, and \citet{pugh2025learning}
presents error-minimization algorithms as Kan extensions in enriched settings.
Building on this line, we invert the characterization into a null hypothesis.
We construct problems whose admissible answers are certified to differ from
both Kan extensions. A calibration phase then checks whether each model's
unconstrained default actually coincides with a Kan extension. \citet{yuan2023power} proves within a stylized categorical model that a task
is prompt-solvable if and only if its task functor is representable. This
result characterizes prompt solvability within its categorical model. Our
instances provide finite, certified completion tasks for empirical evaluation. More recently, \citet{minegishi2026emergent} shows that transformers apply
functor-like correspondences between two instantiated structures. In our
instances, the learner constructs a completion over unobserved objects and
maps, subject to the supplied constraints. Categorical accounts of architectures
\citep{fong2019backprop,gavranovic2024categorical} and of language
\citep{coecke2010mathematical,bradley2022enriched} provide background but do
not yield a capability test.

\paragraph{Formal Abduction and Computational Creativity.} Peircean abduction has also received categorical characterizations
\citep{tohme2015abduction,caterina2016iconicity}, yet the line stops short of a
learner-facing test. Conceptual blending casts creative combination as a colimit, which is a
universal gluing construction
\citep{goguen1999semiotics,eppe2018computational}. Transformational creativity
is modeled as meta-level search \citep{wiggins2006framework}. Bayesian theory
discovery instead searches a grammar of structures \citep{kemp2008discovery}. \citet{mahadevan2026lincs} likewise takes Kan
invariance as the target of learning, asking whether a learned extension
satisfies the universal property. In our test, the canonical
construction supplies a reference completion that the constraints exclude.

\paragraph{Benchmarks for Systematicity and Abduction.}
\citet{phillips2010categorial} explain systematicity by adjunctions, and we
deploy those same operators as foils. Compositional-generalization benchmarks
\citep{lake2018generalization,dziri2023faith} grade convergence toward the
canonical completion of a training fragment. ARC-style tasks
\citep{chollet2019measure,moskvichev2023conceptarc} grade against the rule the
author intended. In neither family is the canonical completion certified as
wrong, so failures are difficult to attribute. In contrast, our instances certify the canonical completion as excluded. Our
Kan-control arm moreover serves as a miniature systematicity test. A
learner that consistently returns its canonical operator passes the
corresponding controls and scores zero on jump instances. These controls
distinguish producing that completion from satisfying the excluding constraints. \citet{meszaros2024rule} formalize rule extrapolation on out-of-distribution
prompts through a simplicity prior. Their normative completion plays the role
that the two Kan extensions play for us. When several simple defaults compete,
our predefined library and the pair-refutation argument of
Section~\ref{sec:scope} absorb the case. A recent survey divides abductive
reasoning in LLMs into generating candidate explanations and selecting among
them \citep{salimi2026wiring}; generation-side systems such as HypoAgent,
which proposes and refines hypotheses over knowledge graphs
\citep{gao2026hypoagent}, work the half that Tier~2 defers, and our
instrument measures the selection half. GEAR evaluates abduction without gold labels by scoring
hypothesis sets for consistency, generalizability, and diversity
\citep{he2025gear}. Our scoring uses structural certification of admissibility,
with answers compared up to renaming. \citet{sun2025occam} measure whether models prefer parsimonious
hypotheses. On our instances the parsimonious completion is the canonical
one and the constraints certify it wrong, so the two measurements pull in
opposite directions by design. \citet{kirkpatrick2025generics} test default
reasoning over linguistic generics, where the default is a usage convention
rather than a certified mathematical object. Closest in spirit, \citet{cooper2026defab} mine public knowledge bases into
poly-time-verifiable defeasible-abduction instances. Each instance asks for a hypothesis that overrides a default while preserving
unrelated expectations. On these instances, frontier models fall far short of
a rule-based solver.
\citet{peng2025abductive} couple multimodal models with an
inductive-logic-programming system that repairs their abduced rules. Both lines score hypotheses with symbolic verifiers, yet in both the rejected
default remains implicit in the source knowledge base. Our instances instead
certify an exact canonical foil together with a computable chance level. A
zero Kan-default rate is therefore attributable to the override itself.

\section{Technical Assumptions and Kan Formulas}\label{app:kan}

The finite-set formulation in the main text is a concrete case of the
following assumptions. They separate the conditions for finite certification
from the structural conditions (J1)--(J4).

\begin{assumption}[Finite Syntax]\label{ass:C}
$\cC$ is given by finite tables of objects, morphisms, identities, and
composition. A presentation used in a prompt denotes these explicit tables.
\end{assumption}
\begin{assumption}[Full Data Inclusion]\label{ass:C0}
$\cCz$ is a nonempty proper full subcategory of $\cC$, with inclusion
$\incl$, and $\Fz\colon\cCz\to\cD$ is the observed functor.
\end{assumption}
\begin{assumption}[Ambient Codomain]\label{ass:D}
We use skeletal $\FinSet$ with size $s([n])=n$. More generally, $\cD$ may be
an effectively presented skeletal category with computable chosen finite
limits and colimits, decidable equality and computable category operations.
Its hom-sets and its objects of size at most $N$ are finite and effectively
enumerable for every $N$, using a computable size function $s$.
\end{assumption}

The complete ambient category supplies the Kan extensions even when their
sizes exceed the allowed answer bound. We impose that bound on solutions
rather than truncating the codomain, since truncation can remove the limits
or colimits needed to define the references.

\begin{assumption}[Constraint Discipline]\label{ass:K}
$\cK$ is a finite list of predicates, polynomial-time decidable on the
instance and extension tables and invariant under gauge equivalence.
Predicates use a fixed, element-blind language of equations, non-equations,
size bounds, and marked diagrams or (co)cones, without constants naming
particular elements or objects of $\cD$.
\end{assumption}
\begin{assumption}[Declared Solution Bound]\label{ass:N}
The learner receives $N\in\N$, and $\cK$ entails $s(F(c))\le N$ at every
new object. Explicit size constraints may supply this entailment.
\end{assumption}
\begin{assumption}[Canon Library]\label{ass:lib}
A finite library $\cL$ of constraint-blind extension operators is fixed before
instance generation. The default library is $\cL_{\Kan}:=\{\Lans,\Rans\}$.
\end{assumption}

Gauge invariance makes a constraint depend on the completed structure.
The element-blind language also prevents the prompt from prescribing an
answer by its element names, although further checks are needed to avoid
prescribing it structurally. Appendix~\ref{app:discipline} describes these
checks and the admissible constraint forms. For general codomains, each
chosen predicate still requires the polynomial-time verification in
Assumption~\ref{ass:K}.

\subsection{Finite Formulas for the Kan Extensions}

We give the concrete finite-set constructions used in
Section~\ref{sec:prelim}. For $c\in\Ob(\cC)$, write
\[
\begin{aligned}
\In(c)&=\{(a,v)\mid a\in\Ob(\cCz),\ v\colon a\to c\},\\
\Out(c)&=\{(b,w)\mid b\in\Ob(\cCz),\ w\colon c\to b\}.
\end{aligned}
\]
These are the object sets of the comma categories $\incl\downarrow c$ and
$c\downarrow\incl$. An incoming-index morphism from $(a,v)$ to $(a',v')$
is $u\colon a\to a'$ in $\cCz$ with $v'\circ u=v$.
An outgoing-index morphism from $(b,w)$ to $(b',w')$ is
$u\colon b\to b'$ in $\cCz$ with $u\circ w=w'$.
All these tables are finite by A1.

The left Kan value is a quotient of a disjoint union of observed sets:
\begin{equation}\label{eq:lan}
\Lan(c)=
\left(\bigsqcup_{(a,v)\in\In(c)}\Fz(a)\right)\!\Big/\!\approx,
\qquad
(a,v'\circ u,x)\approx(a',v',\Fz(u)(x)).
\end{equation}
Here $\approx$ is the equivalence relation generated by the displayed pairs
for all incoming-index morphisms and $x\in\Fz(a)$.
For $f\colon c\to c'$, the function on quotient classes is
$\Lan(f)[a,v,x]=[a,f\circ v,x]$.
Associativity preserves the generating relations and gives functoriality.

The right Kan value is the set of compatible tuples
\begin{equation}\label{eq:ran}
\Ran(c)=
\left\{(y_{(b,w)})\in\prod_{(b,w)\in\Out(c)}\Fz(b)
\ \middle|\
 y_{(b',u\circ w)}=\Fz(u)(y_{(b,w)})\ \forall u
\right\}.
\end{equation}
The condition ranges over all outgoing-index morphisms.
Its action on $f\colon c\to c'$ sends a tuple $y$ to the tuple whose
$(b,w\colon c'\to b)$ coordinate is $y_{(b,w\circ f)}$.
Reindexing respects compatibility, identities, and composition.
An empty incoming diagram gives the empty set, while an empty outgoing
diagram gives a singleton.

Chosen enumerations identify the finite sets in these formulas with their
skeletal representatives $[n]$. At an observed object $c_0$, fullness makes
$(c_0,\id)$ terminal in the incoming category and initial in the outgoing
category. The two values are therefore canonically isomorphic to
$\Fz(c_0)$. Proposition~\ref{prop:wellposed} transports them to strict
extensions and proves independence of the enumerations and limit choices.
For the general codomain of A3, use the corresponding finite comma-category
colimit and limit in place of these set formulas.

\section{Certification Lemmas and Seed Proofs}\label{app:proofs}

\begin{proposition}[Well-posedness of the Canonical Class]\label{prop:wellposed}
Under Assumptions~\ref{ass:C}--\ref{ass:D}, both Kan extensions exist and
are computable as finite comma-category colimits and limits. In $\FinSet$,
Equations~\eqref{eq:lan} and~\eqref{eq:ran} compute their values.
Strict representatives $\Lans,\Rans\in\Ext(\Fz)$ exist, and their gauge
components are independent of the chosen (co)limits and strictification.
\end{proposition}

\begin{proof}[Proof of Proposition~\ref{prop:wellposed}]
The index categories over $\In(c)$ and $\Out(c)$ are finite by
Assumption~\ref{ass:C}. Their chosen colimits and limits are computable by
Assumption~\ref{ass:D}; in $\FinSet$ they are the constructions in
Equations~\eqref{eq:lan} and~\eqref{eq:ran}.
These constructions give the Kan extensions with their action on morphisms.

For $c_0\in\cCz$ the index $(c_0,\id)$ is terminal among $\In(c_0)$: for any
$(a,v)$, fullness (Assumption~\ref{ass:C0}) puts $v$ in $\cCz$ and $u:=v$ is the
unique morphism $(a,v)\to(c_0,\id)$. A colimit over a category with a terminal
object is the value there, giving a natural isomorphism
$\eta_{c_0}\colon\Fz(c_0)\to\Lan(c_0)$. Since $\incl$ is injective on objects,
define $\Lans(c):=\Fz(c)$ for $c\in\cCz$ and $:=\Lan(c)$ otherwise, and transport
morphism tables by conjugation with $\theta_c:=\eta_c$ ($c\in\cCz$), $:=\id$
(new): $\Lans(f):=\theta^{-1}_{\mathrm{cod}}\circ\Lan(f)\circ\theta_{\mathrm{dom}}$.
Functoriality telescopes; for $f$ in $\cCz$, naturality of $\eta$ gives
$\Lans(f)=\Fz(f)$, so $\Lans\circ \incl=\Fz$ on the nose. Dually $\Rans$ via the
counit. Two colimit choices are related by unique comparison isomorphisms
$\gamma_c$ commuting with the colimit injections; $\gamma$ is natural and
$\gamma\circ\eta=\eta'$ since $\eta$ is the injection at $(c_0,\id)$; the induced
isomorphism between the two strictifications has component
$\theta'^{-1}_c\gamma_c\theta_c$, which is the identity on $\cCz$; it is
therefore a gauge equivalence. The strictification itself involves no further choice.
\end{proof}

\begin{proposition}[Gauge Saturation]\label{prop:gauge}
Under Assumption~\ref{ass:K}, $\Adm(S)$ is a union of gauge components; so are
$\ExtN(\Fz)$ and $\Canon_\cL(\Fz)$. Consequently (J1)--(J3), $\chance(S)$, and
all scoring events of Section~\ref{sec:capability} are gauge-invariant.
\end{proposition}

\begin{proof}[Proof of Proposition~\ref{prop:gauge}]
A gauge transformation $\sigma$ is a natural isomorphism $F\cong\sigma\cdot F$
restricting to the identity on $\cCz$, so each $P\in\cK$ takes the same value on
$F$ and $\sigma\cdot F$; gauge fixes object assignments, hence sizes, saturating
$\ExtN$; $\Canon_\cL$ is a union of components by
Definition~\ref{def:canon}.
\end{proof}

\begin{proposition}[Identifiability Implies Finiteness]\label{prop:finite}
For any $F\in\Ext(\Fz)$, the gauge component $[F]$ equals the orbit of $F$ under
$G_F=\prod_{c\notin\cCz}\Aut_\cD(F(c))$; all members share the object assignment;
and $|[F]|\le\prod_{c}\lvert\Aut_\cD(F(c))\rvert<\infty$. Hence under (J3),
$\Adm(S)$ is finite with a well-defined cardinality profile $(n_c)$, and
$\Adm(S)\subseteq\ExtN(\Fz)$ iff $N\ge\max_c n_c$. In $\FinSet$:
$G_F=\prod_c\Sym(n_c)$ and $|[F]|\le\prod_c n_c!$.
\end{proposition}

\begin{proof}[Proof of Proposition~\ref{prop:finite}]
If $\alpha\colon F\Rightarrow F'$ is a natural isomorphism with identity
components on $\cCz$, skeletality (Assumption~\ref{ass:D}) forces
$F'(c)=F(c)$ for all $c$, so $\alpha$ is a gauge tuple and naturality reads
$F'=\alpha\cdot F$. Conversely each $\sigma\in G_F$ is such an isomorphism
$F\Rightarrow\sigma\cdot F$ (conjugation preserves the functor equations). So the
rooted-isomorphism relation is the gauge-orbit relation; finiteness is the
orbit-size bound with finite hom-sets (Assumption~\ref{ass:D}).
\end{proof}

\begin{proof}[Proof of Proposition~\ref{prop:seed}]
An extension is determined by $n=|F(b)|\ge 1$ ($n=0$ is impossible: no map
$[1]\to[0]$), the point $p$ in the image of $F(f)$, and $F(t)$, since
$F(g)\colon[n]\to[1]$ and $F(h)$ are forced. The relations say exactly:
$F(t)(p)=p$, $F(t)^2=\id$, nothing more. Writing $I(m)$ for the number of
involutions on $m$ points ($I(0)=I(1)=1$, $I(2)=2$, $I(3)=4$),
$|\Ext_4|=\sum_{n=1}^{4}n\cdot I(n-1)=1+2+6+16=25$, giving the denominator. For the Kan extensions, the only
morphism $a\to b$ is $f$ (since $t\circ f=f$) and $\Hom(c,b)=\emptyset$, so
$\In(b)=\{(a,f)\}$ and $\Lan(b)=\Fz(a)=[1]$; dually $\Out(b)=\{(c,g)\}$ and
$\Ran(b)=[1]$; all structure maps through $[1]$ are unique, the strict
representatives coincide, and $F(t)=\id_{[1]}$ violates K1 only. For the admissible component, K1
forces a moved point; moved points of an involution come in $2$-cycles; $p$ is
fixed, so $n\ge 3$; K2 caps $n=3$ (and $\cK$ entails the bound $N=4$,
Assumption~\ref{ass:N}). For $n=3$ the admissible $F(t)$ is exactly the
transposition of $[3]\setminus\{p\}$: three tables, one per $p$. Transitivity of
the gauge action $\sigma\cdot F_p=F_{\sigma(p)}$ ($\sigma\in\Sym(3)$) gives one
component (stabilizer of order $2$; orbit $6/2=3$). Minimality is forced, not
stipulated: $n=2$ admits no admissible extension. (J4) strong form:
$\In(b)\neq\emptyset$. All quantities are independently machine-verified by
exhaustive enumeration in the released certifier.
\end{proof}

\begin{proof}[Proof of Proposition~\ref{prop:null}]
(i) $L_\Lambda(S')=\Lambda(\cC,\cCz,\Fz)\in[\Lambda(\cC,\cCz,\Fz)]=\Adm(S')$.
(ii) By (J2) the output lies in $\Canon_\cL$, disjoint from $\Adm$; by (J1),
$\chance(S)>0$. (iii) $L_\Lambda(S^{\circ})=L_\Lambda(S)\in\Canon_\cL(\Fz)$
on every tested input. For independent uniform sampling, conditional on
$S$ and its calibration response, the constrained response remains uniform
on $\ExtN(\Fz)$. Its success probability is $\chance(S)$.
Conditioning further on a canonical calibration response and averaging over
$S$ gives the conditional chance term in (C1).
\end{proof}

\section{The Certified Seed as a Benchmark Item}\label{app:seed}

This appendix renders the instance $S^\ast$ of Proposition~\ref{prop:seed}
using nonce vocabulary and everyday terms. The three stages NARV, QUILB, SORM realize
$a,b,c$; the steps dax, rell, fen realize $f,g,t$; rules R1--R3 are the relations
$t\circ f=f$, $g\circ t=g$, $t^2=\id$; observations fix $\Fz$; requirements
D1--D2 are $\cK$; the answer-format bound realizes $N=4$.

\begin{quote}\small\ttfamily
You are completing the design of a small three-stage signal pipeline.\\[2pt]
The pipeline has three stages: NARV, QUILB, and SORM. Signals flow NARV -> QUILB
-> SORM. There are three processing steps: "dax" carries each NARV state to a
QUILB state; "rell" carries each QUILB state to a SORM state; "fen" is an
internal rearrangement of QUILB.\\[2pt]
Wiring rules (must hold exactly, state by state): R1. Running dax and then fen
lands on the same state as running dax alone. R2. Running fen and then rell gives
the same result as running rell alone. R3. Running fen twice leaves every QUILB
state where it started.\\[2pt]
Recorded observations (fixed): NARV has exactly one state, n1. SORM has exactly
one state, s1. The end-to-end pipeline sends n1 to s1.\\[2pt]
Engineering requirements: D1. QUILB has at most 3 states. D2. fen is NOT the
do-nothing map.\\[2pt]
Answer format bound: designs with up to 4 QUILB states are well-formed answers.
Give one complete implementation: the set QUILB and full tables for dax, rell,
fen.
\end{quote}

An answer is correct precisely when it lies in the single admissible gauge
orbit, whose representative has QUILB $=\{p,q,r\}$ with dax\,$\colon
n1\mapsto p$, rell constant, and fen fixing $p$ while swapping $q$ and $r$. The Kan-default answer is the singleton QUILB with fen the identity; it
reproduces every observation, violates exactly D2, and is the diagnostic
failure mode. The matched control replaces D1--D2 by the single requirement that QUILB has
at most one state, and its unique admissible answer is then the Kan extension
itself.

\section{Generator Integrity Discipline}\label{app:discipline}

The release checks the mathematical instances, their presentation, and
their scores. The implemented checks are as follows.
\begin{enumerate}[leftmargin=1.6em,itemsep=1pt,topsep=2pt]
\item \emph{Constraint-language checks:} predicates refer to structural
properties rather than names of invented elements, and the declared size
bounds make the candidate space finite.
\item \emph{Certification and chance reporting:} smaller configurations
are enumerated, while the family theorem supplies counts for larger
pointed chains. Chance is the fraction of admissible strict tables within
the declared bound. Tests compare the chain counting formulas with
enumeration and check the computed Kan extensions.
\item \emph{Individual constraint checks:} chain certificates report the
fraction of extensions satisfying each constraint separately and flag
fractions below $10\%$. The reported flags remain visible alongside the
certificates; they do not certify propagation or computational hardness.
\item \emph{Presentation controls:} two cover stories preserve the
mathematical instance while changing its vocabulary. The scaled and
asymmetric datasets additionally repeat configurations with fresh nonce
names. Within a rendering, prompt arms share the data and answer format
and replace the requirement block. Line matching does not imply equal
token or required-answer lengths.
\item \emph{Grader checks:} the chain grader is compared with certified
membership on all bounded extensions of the six enumerated original
configurations under both renderings. Renaming, malformed, oversized,
and rule-violating answers are also tested. The asymmetric grader is
checked exhaustively on configurations with at most $8{,}000$ extensions
and on a deterministic sample for the larger configuration.
\item \emph{Output-budget checks:} finish reasons are recorded, unfinished
responses count as failures, and length-limited responses are reported
separately. The budget supplement repeats calibration, jump, and control
calls for four models at $18{,}000$ tokens.
\item \emph{Access and record checks:} provider refusals and checkpoint
serving failures are distinguished from completed evaluation calls.
Recorded model identifiers, prompts, answers, and finish reasons permit
the grades to be replayed. Budgets omitted from older response records
are documented in the run configurations.
\end{enumerate}

\paragraph{Individual-constraint Diagnostics.} The generators plant
solutions and certify (J1)--(J4) and $\chance$ within the declared bound.
Smaller configurations are checked by enumeration; the family theorem
supplies the counts and certificates for larger chains.
Appendix~\ref{app:discipline} records the implemented checks on the
constraints, rendering, grading, and evaluation conditions. The chain
certifier also reports the fraction of bounded extensions satisfying each
constraint separately, using $10\%$ as a diagnostic flag. Six of the nine
original configurations exceed this fraction on every constraint.
The configurations $(2,(2,3))$, $(3,(2,3,3))$, and $(3,(3,3,3))$ fall
below it on one or two size constraints, down to $1.4\%$. Their certificates
and flags are retained; the conjunction still determines the unique
admissible component.

The implemented checks support the evaluations reported here. Length-matched
answer alternatives, indirect-exclusion controls, and comparisons with
matched generic solvers would address further questions; they are not
reported as completed experiments.

\section{Proof of the Family Theorem}\label{app:family}

The proof has two parts. The local cycle constraints determine every
admissible structure, while the maps between successive objects give a
recursion for all bounded extensions. Their ratio yields the chance level
without enumerating those extensions.

\paragraph{The Finite Category.}
Order the objects as $a<b_1<\cdots<b_m<c$. Between any two distinct objects
$x<y$, there is one morphism $P(x,y)$; there are no backward morphisms.
The endomorphisms of $b_i$ form the cyclic group
$\{\id,t_i,\ldots,t_i^{p_i-1}\}$, and $a,c$ have identity endomorphisms.
Composition adds exponents modulo $p_i$ within a group, composes forward
paths to their unique forward path, and absorbs endomorphisms adjacent to a
forward path. These rules are associative. A composable triple with distinct
endpoints has a unique possible composite; a triple with equal endpoints
lies in a cyclic group. The full data subcategory on $a,c$ contains the
unique path $P(a,c)$.

\paragraph{Parametrizing Extensions.}
Write $\varphi_0$ for the map from $a$ to $b_1$, $\varphi_i$ for the map
from $b_i$ to $b_{i+1}$, and $\varphi_m$ for the map from $b_m$ to $c$.
An extension is determined by sizes $n_i\ge1$, permutations
$\tau_i\colon[n_i]\to[n_i]$, and these connecting maps, subject to
\[
\tau_i^{p_i}=\id,\qquad
\tau_i\circ\varphi_{i-1}=\varphi_{i-1},\qquad
\varphi_i\circ\tau_i=\varphi_i.
\]
The incoming image consists of fixed points of $\tau_i$, and the outgoing
map is constant on its cycles. Every such tuple defines a functor by
composing the connecting maps on longer paths. Its restriction to the
singleton data is automatic. Conversely, the composition table forces every
extension to have this form. Since $p_i$ is prime, each cycle of $\tau_i$
has length $1$ or $p_i$.

\paragraph{Admissible Structure and Its Orbit.}
The nonidentity constraint requires a $p_i$-cycle. The incoming map has
nonempty domain and supplies a fixed point, so $n_i\ge1+p_i$.
The size constraint gives the reverse inequality. Thus every admissible
$\tau_i$ has one fixed point $q_i$ and one $p_i$-cycle. The incoming map
is constant at $q_i$. This also forces each outgoing map to be constant,
since its image lies in the next object's singleton fixed-point set or in
$F(c)=[1]$. Conversely, these assignments satisfy every relation and
constraint, proving solvability and the structural characterization.

There are $1+p_i$ choices for $q_i$ and $(p_i-1)!$ cycles on its complement.
The connecting maps introduce no further choices, giving
$|\Adm|=\prod_i(1+p_i)(p_i-1)!$.
Given two admissible extensions, map their fixed points to each other and
map one cycle to the other in cyclic order at each object. These bijections
conjugate the endomorphisms and commute with the constant connecting maps.
They give a gauge equivalence, proving that the admissible set is one orbit.
Its stabilizer has order $\prod_i p_i$, agreeing with the direct count.

\paragraph{Canonical Exclusion and Support.}
At each $b_i$, there is a single incoming route from the data and a single
outgoing route to the data. Each corresponding comma category has one
object and its identity, so both Kan values are $[1]$.
Their strict representatives are the all-singleton extension. It violates
every $\mathrm{K1}_i$ and satisfies every $\mathrm{K2}_i$.
Each $b_i$ receives a path from $a$, giving the strong form of (J4).
The constraints entail the declared bound whenever
$N\ge\max_i(1+p_i)$. This completes (J1)--(J4).

\paragraph{The Bounded Count.}
To count all extensions, allow $1\le n\le N$ and let $k$ be the number of
fixed points of an endomorphism of prime order dividing $p$.
The number of permutations with these parameters is
\begin{equation}\label{eq:cycle-weight}
w_p(n,k)=
\begin{cases}
\displaystyle\frac{n!}{k!\,p^j j!},&j=(n-k)/p\in\mathbb Z_{\ge0},\\[4pt]
0,&\text{otherwise}.
\end{cases}
\end{equation}
Choosing the fixed set and arranging the remaining elements into $j$
unordered $p$-cycles gives this formula. The factor $p^j$ accounts for
cyclic rotations, and $j!$ accounts for permutations of the cycles.
An endomorphism with these parameters has $k+j$ orbits.

For fixed endomorphisms at successive objects, the first map has $k_1$
choices. A map from object $i$ to object $i+1$ chooses one target fixed
point independently for each source orbit, giving
$k_{i+1}^{\,k_i+(n_i-k_i)/p_i}$ choices. The final map to $[1]$ is unique.
Consequently
\[
|\ExtN(\Fz)|=
\sum_{(n_i,k_i)_{i=1}^m}
\left(\prod_{i=1}^m w_{p_i}(n_i,k_i)\right)
k_1\prod_{i=1}^{m-1}
k_{i+1}^{\,k_i+(n_i-k_i)/p_i},
\]
where the sum includes $1\le n_i\le N$, $0\le k_i\le n_i$, and
$p_i\mid(n_i-k_i)$. Terms with no fixed-point choices contribute zero.
Grouping the sum by its last object gives the transfer recursion
\begin{equation}\label{eq:transfer}
\begin{aligned}
R_1(n,k)&=w_{p_1}(n,k)\,k,\\
R_{i+1}(n',k')&=w_{p_{i+1}}(n',k')
\sum_{(n,k)}R_i(n,k)(k')^{\,k+(n-k)/p_i},\\
|\ExtN(\Fz)|&=\sum_{(n,k)}R_m(n,k).
\end{aligned}
\end{equation}
Each stage has $O(N^2)$ states and at most $O(N^4)$ transitions.
Precomputing the weights and integer powers gives $O(mN^4)$ arithmetic
operations. This counts arithmetic operations on integers, whose bit lengths
increase with the parameters; it is not a unit-cost claim about total runtime.

\paragraph{Decay of the Uniform Chance Level.}
A one-object bounded extension is a triple $(n_i,\tau_i,q_i)$ with $q_i$
a fixed point. A tuple of these one-object extensions defines a chain
extension by taking every connecting map to be constant at the next $q_i$.
All domains are nonempty, so the connecting maps recover the chosen $q_i$;
this construction is injective. Hence the chain denominator is at least the
product of the one-object denominators. The admissible numerator is their
product, and therefore
\[
\chance(S(m,\vec p\,))
\le\prod_{i=1}^m\chance(S(1,(p_i))),
\]
using the same bound $N$ throughout. For $p_i=2$, $N=4$, each factor is
$3/25$ by Proposition~\ref{prop:seed}, proving
Equation~\eqref{eq:family-chance}.

\paragraph{Composite Orders.}
If $p_i$ is composite, choose a proper divisor $d>1$. One fixed point and
one $d$-cycle satisfy $\tau_i^{p_i}=\id$ and both local constraints, as do
one fixed point and one $p_i$-cycle. Constant connecting maps make either
choice an admissible extension. The distinct size profiles are not
gauge-equivalent, so (J3) fails. \qed

\section{Proofs for Section~\ref{sec:beyond}}\label{app:beyond}

We use the following conjugation identity throughout. For a functor $H$ and
any family of objectwise automorphisms $\beta_x$, set
$(\beta\cdot H)(f\colon x\to y)=\beta_yH(f)\beta_x^{-1}$.
For composable $f,g$, the two middle automorphisms cancel, so
$(\beta\cdot H)(g)(\beta\cdot H)(f)=\beta_zH(gf)\beta_x^{-1}$.
Thus $\beta\cdot H$ is a functor, $\beta$ is a natural isomorphism
$H\Rightarrow\beta\cdot H$, and restriction to a subcategory on which
$\beta$ is the identity leaves every table unchanged.

\begin{proof}[Proof of Proposition~\ref{prop:predchar}]
If $\widetilde F$ and an isomorphism $\alpha\colon\widetilde F\Rightarrow G$
witness predictive correctness, restriction to $\cC$ gives a
$\cCz$-rooted isomorphism $F\cong G|_{\cC}$. Hence
$F\in[G|_{\cC}]$. Conversely, write $F=\sigma\cdot(G|_{\cC})$ for a
gauge transformation $\sigma$, and extend $\sigma$ by identities on the
held-out objects. The resulting functor $\widetilde F=\widehat\sigma\cdot G$
restricts to $F$, and $\widehat\sigma^{-1}\colon\widetilde F\Rightarrow G$
is the required rooted isomorphism. This proves the characterization and
gauge invariance.

Under soundness, $G|_{\cC}$ is admissible, so gauge saturation gives
$[G|_{\cC}]\subseteq\Adm(S)$. If (J3) holds, the containing admissible
component equals $[G|_{\cC}]$. Conversely, if every admissible completion
is predictively correct, the characterization gives the reverse containment,
so $\Adm(S)=[G|_{\cC}]$ and (J3) holds.
\end{proof}

\begin{proof}[Proof of Proposition~\ref{prop:separation}]
An extension of the relaxed seed is determined by
$n=|F(b)|\ge1$, the point $p=F(f)(1)$, and the involution $\tau=F(t)$
fixing $p$. Admissibility requires $\tau\neq\id$ and $n\le4$.
A moved point needs a transposition disjoint from $p$, so $n\ge3$.
At $n=3$, $\tau$ swaps the two other points, yielding three tables indexed
by $p$. At $n=4$, it fixes one additional point $q\neq p$ and swaps the
remaining pair, yielding $4\cdot3=12$ tables indexed by $(p,q)$.
Permutations act transitively on each collection, while their different
sizes preclude a gauge equivalence. These are exactly $Y_3$ and $Y_4$.

We specify the larger category by its complete finite composition table.
Adjoin objects $u,v$ to the seed category, their identities, and arrows
\[
\mu\colon u\to b,\quad \rho\colon b\to v,\quad
\gamma\colon u\to c,\quad \psi\colon a\to v,\quad
\theta\colon u\to v.
\]
Besides identity compositions, all composable pairs of nonidentity arrows
are listed below; the row arrow is composed after the column arrow:
\[
\begin{array}{c|ccc}
\circ & f&t&\mu\\ \hline
t&f&\id_b&\mu\\
g&h&g&\gamma\\
\rho&\psi&\rho&\theta
\end{array}
\]
Every composable triple of nonidentity arrows has middle arrow $t$.
Associativity follows from $t^2=\id_b$, the relations $tf=f$, $t\mu=\mu$,
and $gt=g$, $\rho t=\rho$ in the table. Thus this defines a category with
five objects and fourteen morphisms. Its restriction to $\{a,b,c\}$ is
exactly the seed category, so the inclusion is full.

Define the ground truth by
\[
G(a)=G(c)=G(u)=[1],\qquad G(b)=[4],\qquad G(v)=[2],
\]
with $G(f)(1)=1$, $G(\mu)(1)=2$, $G(t)=(3\ 4)$ fixing $1,2$,
and
\[
G(\rho)(2)=2,\qquad
G(\rho)(1)=G(\rho)(3)=G(\rho)(4)=1.
\]
The map $G(g)$ is unique, and the remaining tables are the composites in
the displayed table. Its absorption relations hold because both chosen
points are fixed and $G(\rho)$ is constant on the transposed pair.
Hence $G$ is a functor and $G|_{\cC}\in Y_4$, establishing soundness.

To specify precisely the observations, let $\mathcal O$ be the full subcategory
on $\{a,c,u,v\}$ and reveal $G|_{\mathcal O}$. Its nonidentity arrows are
$h,\gamma,\psi,\theta$; in particular,
$G(\psi)(1)=1$ and $G(\theta)(1)=2$.
For any $F\in Y_4$, let $p$ be its incoming point and $q$ its other fixed
point. Extend it by $\widetilde F(\mu)(1)=q$ and by the map
$\widetilde F(\rho)$ sending $q$ to $2$ and every other point to $1$.
Together with the same held-out sets and the forced composite tables, these
assignments define a functor agreeing with $G$ on $\mathcal O$.

Conversely, suppose an extension $\widetilde F$ of $F$ agrees with those
observations. The point $q=\widetilde F(\mu)(1)$ is fixed by $F(t)$
because $t\mu=\mu$. Functoriality and the recorded values imply
\[
\widetilde F(\rho)(p)=1,\qquad
\widetilde F(\rho)(q)=2.
\]
Thus $q\neq p$, so $F(t)$ has at least two fixed points. No member of
$Y_3$ can satisfy the observations. This is an observable refutation,
separate from the structural characterization: Proposition~\ref{prop:predchar}
and $G|_{\cC}\in Y_4$ show that exactly $Y_4$ is predictively correct.
\end{proof}

Soundness is essential to the preceding conclusion. If the original tight
seed is paired with this four-element ground truth, its unique admissible
answer still disagrees with the world, because its size constraint is false
there. Our distinction therefore separates ambiguity under valid constraints
from error in the supplied constraints themselves.

\begin{proof}[Proof of Theorem~\ref{thm:chainwd}]
The successor category is finite and its data inclusion is proper and full
by definition. The codomain is unchanged, and the successor constraints
satisfy A4 and entail the new bound. These are A1--A5 for the incorporated
instance.

Let $X'=\sigma\cdot X$ be two gauge-equivalent commitments. Regard
$\sigma\colon X\Rightarrow X'$ as an isomorphism of the successor data,
and extend it by identities on $\cC'\setminus\cC$. Conjugation gives
\[
c_\sigma\colon\Ext(X)\longrightarrow\Ext(X'),\qquad
H\longmapsto\widehat\sigma\cdot H.
\]
Its inverse is $c_{\sigma^{-1}}$, and it preserves object assignments.
Successor constraints are invariant under this renaming, including at
the observed objects, so $c_\sigma$ maps admissible sets bijectively onto
one another. It also preserves bounded extension sets. A successor gauge
tuple is the identity on $\cC$, whereas $\widehat\sigma$ is the identity
outside $\cC$; these families commute. Therefore $c_\sigma$ carries each
successor gauge component bijectively onto a successor gauge component.
This proves equality of component sizes, bounded counts, and, when the
denominator is nonzero, chance. It also proves invariance of (J1) and
(J3); (J4) depends only on the shared category and inclusion.

For (J2), the isomorphism $\sigma$ induces isomorphisms between the
pointwise comma-category diagrams defining the Kan extensions of $X$ and
$X'$. Their universal properties supply natural comparison isomorphisms
between the Kan extensions. After strictification, the comparison at each
data object is $\sigma$. Composing it with $\widehat\sigma^{-1}$ therefore
gives a data-rooted isomorphism from the transported Kan extension of $X$
to the Kan extension of $X'$. Consequently
\[
c_\sigma\bigl(\Kan(X)\bigr)=\Kan(X'),
\]
and disjointness from the admissible set is preserved.
More generally, the same argument applies to a library whose operators
have isomorphism-invariant domains and satisfy
$[\Lambda(X')]=[c_\sigma(\Lambda(X))]$ whenever they are defined.
This is the required respect for data isomorphisms; constraint-blindness
alone does not imply it.
\end{proof}

To compare admissibility at different stages, first express the conditions
on one common category. Write $\cC_j$ for the completed category at stage
$j$. At a later stage $\ell$, let
\[
\mathcal A_j^{(\ell)}:=
\{H\colon\cC_\ell\to\cD\mid H|_{\cC_j}\in\Adm(S_j)\},
\qquad j\le\ell.
\]
These sets impose the earlier admissibility conditions, including their
committed observations, on candidate tables for the later category.

\begin{lemma}[Monotonicity under Fixed Commitments]\label{lem:mono}
Along a committed chain,
$\mathcal A_1^{(\ell)}\supseteq\cdots\supseteq
\mathcal A_\ell^{(\ell)}=\Adm(S_\ell)$.
\end{lemma}

\begin{proof}[Proof of Lemma~\ref{lem:mono}]
Write $X_j$ for the commitment at stage $j$. If
$H\in\mathcal A_{j+1}^{(\ell)}$, its restriction to $\cC_{j+1}$ is
admissible at that stage and hence restricts to the observed functor
$X_j$ on $\cC_j$. Since $X_j\in\Adm(S_j)$, this gives
$H\in\mathcal A_j^{(\ell)}$. Applying this argument for each $j<\ell$
gives the inclusions. Restriction to $\cC_\ell$ itself is the identity,
so the last set equals $\Adm(S_\ell)$.
\end{proof}

\paragraph{Retrospective correctness and stabilization.}
Restriction of a stage-$(\ell+1)$ witness to $\cC_\ell$ retains all
earlier constraints and its stage-$k$ commitment. Thus
$\mathcal R_{\ell+1}(k)\subseteq\mathcal R_\ell(k)$, with
$\mathcal R_k(k)=\Adm(S_k)$.
If $X$ has a witness $H$ and $X'=\sigma\cdot X$ is stage-$k$
gauge-equivalent to it, extend $\sigma$ by identities to $\cC_\ell$.
Conjugating $H$ yields a witness for $X'$: the fixed earlier observations
are preserved, and every successor constraint is invariant under renaming
at all its objects. Hence each $\mathcal R_\ell(k)$ is a union of
stage-$k$ gauge components.

The declared bound makes $\Adm(S_k)$ finite. Each answer that ever leaves
the decreasing sequence has a first elimination stage. There are finitely
many such answers, so after the largest of these stages no further change
occurs. The eventual value is the intersection defining ultimate
correctness. This gives no bound on when the last elimination occurs.
Under (J3), gauge saturation leaves only the possibilities
$\mathcal R_\ell(k)=\varnothing$ and
$\mathcal R_\ell(k)=\Adm(S_k)$, proving the common continuation verdict.

\begin{proof}[Proof of Theorem~\ref{thm:entrench}]
Take the relaxed seed as stage one, whose two admissible components were
established in Proposition~\ref{prop:separation}. Adjoin one object $d$,
its identity, and arrows $r\colon d\to b$ and $r_c\colon d\to c$.
Besides the old and identity compositions, impose only
$t\circ r=r$ and $g\circ r=r_c$. Associativity follows from the old
relations $t^2=\id_b$ and $gt=g$; no new arrow has both endpoints in the
old category. This is a finite proper full extension.

The successor constraints say that the image of the marked cone
$(d;r,r)$ over the parallel pair $(t,\id_b)$ is limiting, and that
$|F(d)|\le1$; take $N'=1$. Both constraints are invariant under all
objectwise renamings, and the latter entails the bound.
For a commitment $X$, write
$\Phi=\{x\in X(b):X(t)(x)=x\}$.
Every successor extension is specified by $q=|F(d)|\ge0$ and
$F(r)\colon[q]\to X(b)$ with image in $\Phi$; the table of $r_c$ is
forced by composition. The inclusion $\Phi\hookrightarrow X(b)$ is
the equalizer of $X(t)$ and the identity, since every map equalizing
them factors uniquely through that inclusion. Therefore the marked
cone is limiting exactly when $F(r)$ is a bijection $[q]\to\Phi$.
This condition is decidable directly from the two finite tables, so the
constraints also satisfy the evaluation requirement in A4.

For $X\in Y_3$, $\Phi$ is its single incoming fixed point. The
constraints force $q=1$ and the unique map selecting that point,
giving exactly one admissible successor table. To identify it with the
right Kan extension, the outgoing routes at $d$ are $(b,r)$ and
$(c,r_c)$. Compatibility with $t$ forces the $b$-coordinate to lie in
$\Phi$, and compatibility with $g$ determines the $c$-coordinate.
Thus the right Kan value at $d$ is $\Phi$, and in this singleton case
its strict representative is exactly the unique admissible table.
The successor has (J1), (J3), and (J4), since $\Out(d)\neq\varnothing$,
and its admissible set is the right Kan component. It is therefore a
control instance targeting $\Rans$.

For $X\in Y_4$, $|\Phi|=2$. The limit constraint forces $q=2$,
contradicting $q\le1$, so no admissible successor exists. A two-stage
continuation exists exactly for $X\in Y_3$, which proves
$\mathcal R_2(1)=Y_3$.
\end{proof}
\section{Experimental Details}\label{app:exp}

\paragraph{Plain-language Pipeline.} One instance travels the following path.
The generator writes the composition tables, the observed functor, and the
constraint list. The certifier either enumerates every extension within the
declared bound and checks the four conditions directly, or, on the
theorem-certified tier, verifies the hypotheses of Theorem~\ref{thm:family}
and reads off the counts. The renderer wraps the tables in a nonce cover
story, the grader parses an answer back into a candidate extension and
reports one of five outcomes, and the harness records the response and
finish reason. The released instances and raw answers permit this pipeline
to be checked independently of the model's performance.

\paragraph{Sensitivity to the Size Bound.} The declared bound determines
which candidate tables are permitted and enters the prompt, grader, and
chance denominator. At $m=2$, $p=(2,2)$, the extension count grows from
$4{,}387$ at $N=4$ to $117{,}867$ at $N=5$, and chance falls from
$2.1\times10^{-3}$ to $7.6\times10^{-5}$. The unique admissible structure
fits every $N\ge\max_i(1+p_i)$, while the generator's default is
$N=\max_i(1+p_i)+1$. The instance $(2,(3,2))$ already uses the minimum
$N=4$. Enlarging the candidate bound does not change the admissible
structure, because each jump size constraint still permits at most
$1+p_i$ elements. Thus the bound changes the uniform-table baseline
without by itself establishing a change in computational difficulty.
Padding an admissible answer is excluded by these same size constraints.

\paragraph{Models and Access.} The four API models were accessed through a
single aggregator on August 20--21, 2026. The ten open checkpoints in
Table~\ref{tab:main} were served locally on an RTX PRO~6000 Blackwell with
96\,GB of memory; run notes record vLLM versions 0.27.1 and 0.28.0.
Llama~3.3~70B and Qwen2.5~72B use the RedHatAI FP8 checkpoints identified
in the response logs. The runs fix temperature and retain checkpoint
sampling defaults for other parameters. For example, Llama uses
top-$p=0.9$, while Qwen3 uses top-$p=0.95$ and top-$k=20$.
Comparisons therefore describe the evaluated model and decoding
configuration together. Three attempted Gemma checkpoints were excluded
after kernel failures prevented evaluation. An earlier provider access
configuration also refused the task through content filtering; those
setup failures are separate from the reported evaluation calls.

\paragraph{Sampling.} Table~\ref{tab:sampling} gives the schedules for the
original evaluations and their supplements. Each prompt variant receives
one greedy call followed by temperature-$0.7$ calls, so the single scaled
control call is greedy. Each call starts a fresh conversation. The API
wording supplement adds four jump calls on each of two instances under
the second cover story, without calibration or control under that story.
The table in the main text retains those eight calls per API model in its
jump counts. Open chain and asymmetric evaluations use both cover stories;
scaled and scratchpad evaluations use the first story.

\begin{table}[htbp]
\centering
\caption{Sampling schedules. Calls per cell are ordered as
jump/calibration/control; asymmetric cells have separate left and right
controls. A cell is one model, instance, and rendering. API wording calls
are included with the API chain counts in Table~\ref{tab:main}.}
\label{tab:sampling}
\small
\setlength{\tabcolsep}{4pt}
\begin{tabular}{lrrrrrr}
\toprule
Evaluation & Models & Instances & Stories & Calls/cell & Total & Jump\\
\midrule
API chain & 4 & 9 & 1 & 6/4/2 & 432 & 216\\
API wording & 4 & 2 & 1 & 4/0/0 & 32 & 32\\
Open chain & 10 & 9 & 2 & 6/4/2 & 2,160 & 1,080\\
Scaled chain & 10 & 180 & 1 & 3/2/1 & 10,800 & 5,400\\
Asymmetric & 10 & 70 & 2 & 4/3/2/2 & 15,400 & 5,600\\
\midrule
Standard evaluations & & & & & 28,824 & 12,328\\
Scratchpad & 10 & 9 & 1 & 6/4/2 & 1,080 & 540\\
18,000-token budget & 4 & 9 & 2 & 6/4/2 & 864 & 432\\
\midrule
All recorded evaluations & & & & & 30,768 & 13,300\\
\bottomrule
\end{tabular}
\end{table}

The standard evaluations and scratchpad supplement use a $12{,}000$-token
output budget. The budget supplement repeats all three chain arms for
four open models at $18{,}000$ tokens. Later logs record the budget per
response; for $1{,}976$ older responses it is documented in the evaluation
commands and run notes. A released pilot of $165$ answers used an earlier
harness with a $6{,}000$-token budget and without recorded finish reasons, and is excluded
from all reported counts. The $259$ dataset entries comprise the nine
original chain configurations, eighteen scaled configurations with ten
nonce vocabularies each, and seven asymmetric configurations with ten
vocabularies each. The original nine chain configurations recur in the
scaled set, giving eighteen distinct chain configurations overall.

\paragraph{Grading.} Answers are classified as admissible, Kan default,
other valid completion, rule-violating, or invalid. The chain grader uses
the structural characterization of Theorem~\ref{thm:family}(i); the
asymmetric grader checks its corresponding structural predicates.
The chain grader agrees with certified membership on all $138{,}691$
enumerated extensions of the six released $m\le2$ instances under both
renderings, and passes the renaming and malformed-answer checks.
An answer with a recorded non-stop finish reason is invalid even when
an intermediate design can be parsed. The chain and scratchpad parser
uses the last complete answer block; the asymmetric grader uses its own
table parser. Replaying the current graders and finish policy reproduces
all $30{,}768$ recorded classes, with unique sample keys. Earlier
scratchpad extraction repairs affect $88$ responses, whose earlier
classes remain available in the logs.

\paragraph{Aggregates.} On the original chain evaluation, the API models
return canonical calibration answers in $129/144$ calls; conditioning on
responses that are neither length-limited nor empty gives $129/132$.
For the eight open models with reasoning traces, the corresponding counts
are $297/576$ and $297/352$. Controls succeed in $69/72$ API calls,
$227/288$ reasoning-model calls, and $72/72$ calls from the two models
answering directly. The main table uses unconditional calibration rates.
The conditional fractions describe the subset of completed responses.

Across all recorded evaluations, $7{,}135/13{,}300$ jump answers are
admissible. The remaining $6{,}165$ comprise $3{,}571$ invalid responses,
$2{,}588$ rule violations, five other valid completions, and one canonical
completion. Of the invalid responses, $3{,}541$ have finish reason length.
A parseable intermediate table exists in $1{,}439$ non-stop jump responses;
these remain invalid under the completion policy. The standard evaluations
alone contain $6{,}523/12{,}328$ admissible jump answers and the same single
canonical answer. Scratchpad and increased-budget calls contribute
$259/540$ and $353/432$ admissible answers, respectively, and zero canonical
answers.

\paragraph{Matching Calibration and Success.}\label{app:matched}
To estimate the conditional quantity in (C1), we match independent
calibration and jump calls by model, instance, rendering, and temperature.
Let $\widehat d_i$ be the fraction of canonical calibration calls in
cell $i$, and $\widehat a_i$ the fraction of admissible jump calls.
Giving each instance and rendering equal weight before conditioning gives
\begin{equation}\label{eq:empirical-matched}
\widehat A_{\mathrm{cal}}=
\frac{\sum_i\widehat d_i\widehat a_i}{\sum_i\widehat d_i},\qquad
\widehat B_{\mathrm{cal}}=
\frac{\sum_i\widehat d_i\chance(S_i)}{\sum_i\widehat d_i}.
\end{equation}
The success estimate and its uniform-table baseline use the same
calibration weights. Both are undefined when the canonical calibration count
is zero. Table~\ref{tab:matched} gives the temperature-$0.7$
estimates; API wording cells without calibration are excluded from this
calculation. The main figures retain their original pooled call
fractions, whose greedy proportions differ across arms. We do not use
those pooled fractions as the same-policy conditional estimate.
These finite-sample estimates describe the evaluated cells, without
assigning a retrospective (C0) pass threshold.

\begin{table}[htbp]
\centering
\caption{Calibration-weighted accuracy at temperature $0.7$, with the matched
uniform-table baseline in parentheses. Undefined entries have a canonical
calibration count of zero; a dash indicates an unevaluated family. The scaled
control arm is greedy and is reported separately in the main experiment.}
\label{tab:matched}
\small
\setlength{\tabcolsep}{5pt}
\begin{tabular}{lccc}
\toprule
Model & Original chain & Scaled chain & Asymmetric\\
\midrule
GPT-5.6 Luna Pro & 0.954 (0.030) & -- & --\\
Claude Sonnet 5 & 0.978 (0.029) & -- & --\\
Gemini 3.1 Pro & 1.000 (0.030) & -- & --\\
DeepSeek V4 Pro & 0.800 (0.036) & -- & --\\
Qwen3.6 35B-A3B & 0.171 (0.070) & 0.038 (0.059) & 0.333 (0.114)\\
Qwen3 32B & 0.971 (0.033) & 0.886 (0.016) & 1.000 (0.090)\\
Qwen3.6 27B & \textit{undefined} & 0.167 (0.082) & 0.744 (0.108)\\
gpt-oss 20B & 0.974 (0.027) & 0.965 (0.014) & 0.982 (0.091)\\
Qwen3 14B & 0.916 (0.026) & 0.872 (0.015) & 1.000 (0.089)\\
Qwen3.5 9B & 0.320 (0.030) & 0.250 (0.015) & 0.926 (0.084)\\
Qwen3 8B & 0.747 (0.029) & 0.634 (0.015) & 0.997 (0.088)\\
Qwen3.5 4B & 0.620 (0.074) & 0.643 (0.077) & 0.922 (0.096)\\
Qwen2.5 72B & 0.000 ($2.15\times10^{-5}$) & 0.000 (0.026) & 0.010 (0.087)\\
Llama 3.3 70B & 0.000 (0.003) & 0.028 (0.005) & 0.019 (0.115)\\
\bottomrule
\end{tabular}
\end{table}

\paragraph{Intervals and Figure Populations.} The Wilson intervals in the
figures are nominal binomial intervals around the displayed pooled call
fractions. They summarize the evaluated collection and do not quantify
generalization to new models or new mathematical configurations.
Figure~\ref{fig:instrument} pools original and scaled chain calls under
both available cover stories, including the API wording calls.
Figure~\ref{fig:difficulty} uses the first cover story and groups by
configuration; the nine repeated open-model configurations each contain
six original and thirty scaled calls, while the nine additional
configurations have thirty calls each. API points have six calls per
configuration. Appendix figure captions specify the other populations.

\paragraph{Asymmetric Instances.}\label{app:asym}
The asymmetric generator uses \mbox{$a\xrightarrow{f}b\xrightarrow{g}c$}
with an involution $t$ on $b$, satisfying $tf=f$ and $gt=g$. The observed
sets at $a$ and $c$ have sizes $k$ and $\ell$, and the observed composite
is constant. The left and right Kan extensions have $k$ and $\ell$
elements at $b$, respectively, and both act identically under $t$.
The jump constraints require nonidentity $F(t)$, injective $F(f)$,
surjective $F(g)$, and $|F(b)|\le N$. Enumeration verifies a unique
admissible gauge class excluding both Kan extensions for
$(k,\ell,N)\in\{(2,1,4),(3,1,5),(1,2,3),(2,2,4),(3,2,5),(4,1,6),(5,1,7)\}$.
The left control pins the left Kan extension through injectivity and
$|F(b)|\le k$; the right control pins the right extension through
surjectivity and $|F(b)|\le\ell$.

\paragraph{Output Budget.}
The increased-budget run repeats calibration and control alongside jump
calls. At $18{,}000$ tokens, Qwen3.6~27B returns canonical calibration
answers in $4/72$ calls despite reaching $76/108$ admissible jump answers.
Qwen3.5~4B's controls change from $32/36$ to $28/36$. These independent
reruns show budget sensitivity alongside sampling variation.

\begin{table}[tbp]
\centering
\caption{What the output budget costs, on the pointed-chain family. Each entry counts
admissible answers out of $108$ constrained samples, with length-limited answers in brackets.
The Kan-default rate stays at $0$ of $432$ at both budgets.}
\label{tab:budget}
\small
\begin{tabular}{lcc}
\toprule
\textbf{model} & \textbf{12,000 tokens} & \textbf{18,000 tokens}\\
\midrule
Qwen3 8B & 83 \,[20] & 95 \,[2]\\
Qwen3.5 9B & 44 \,[64] & 97 \,[11]\\
Qwen3.5 4B & 30 \,[78] & 85 \,[19]\\
Qwen3.6 27B & 20 \,[88] & 76 \,[32]\\
\bottomrule
\end{tabular}
\end{table}

\paragraph{Likelihood Scores.}
The teacher-forced scorer evaluates one serialized representative per
candidate, disables template thinking where supported, and sums token
log probabilities without a terminal stop token. On the original chains,
the admissible candidate uses $2.05$--$2.94$ times as many tokens as the
canonical candidate. The canonical string wins $174/180$ comparisons
under the summed score and $132/180$ under the mean per-token score.
This sensitivity concerns candidate strings, not the probability mass of
entire gauge classes. On the equal-size asymmetric configuration, the
two Kan strings have equal token counts, and the left wins $142/200$
likelihood comparisons. Its sampled calibration counts are $228$ left
and $30$ right Kan answers out of $600$ calls, with $192$ invalid,
$146$ other valid, and four rule-violating responses.

\section{Additional Figures}\label{app:figs}

The following figures provide the complete model comparisons and
supplementary analyses referenced in Section~\ref{sec:experiments}.

\begin{figure}[htbp]
\centering
\includegraphics[width=\textwidth]{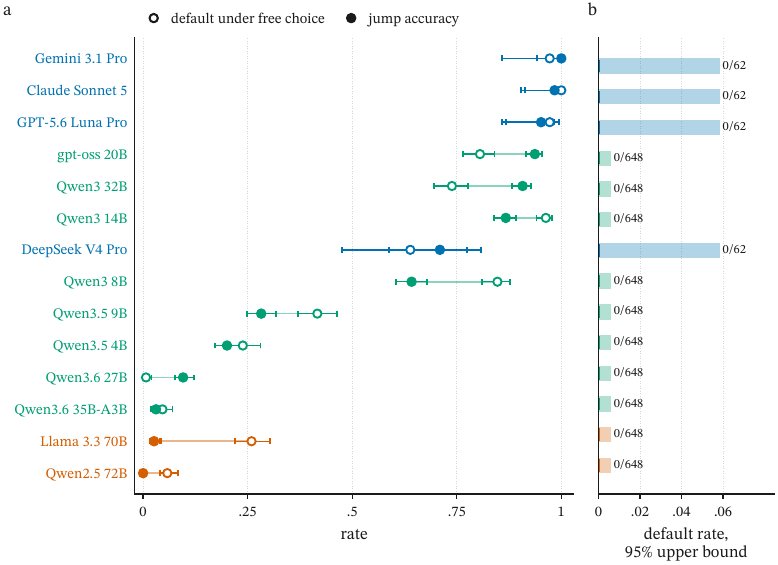}
\caption{Canonical calibration and jump accuracy pooled over the original and scaled
chains. Each API model contributes $36$ calibration and $62$ jump answers, including
its wording supplement; each open model contributes $432$ and $648$. Left: observed
rates over all calls, with nominal Wilson $95\%$ intervals. Right: the upper endpoint
for the canonical-response rate under the constraints, with event counts beside each bar.}
\label{fig:instrument}
\end{figure}

\begin{figure}[tbp]
\centering
\includegraphics[width=0.74\textwidth]{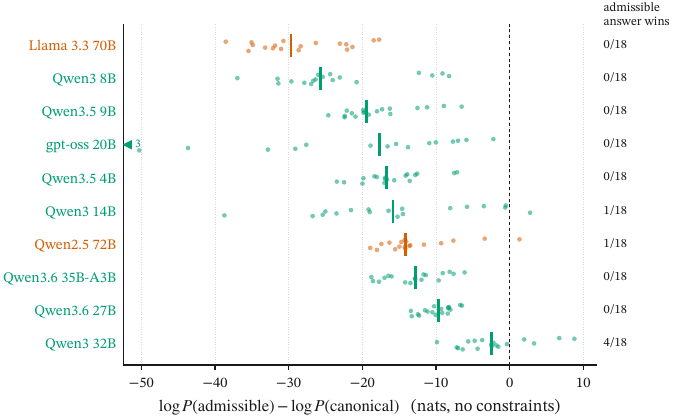}
\caption{Summed teacher-forced token log probability of one admissible answer string
minus that of one canonical string under the calibration prompt in immediate-answer mode.
Each open model has $18$ instance--rendering pairs. The bar is the median, the count at
right records admissible-string wins, and a triangle marks points beyond the axis.}
\label{fig:logprob}
\end{figure}

\begin{figure}[tbp]
\centering
\includegraphics[width=\textwidth]{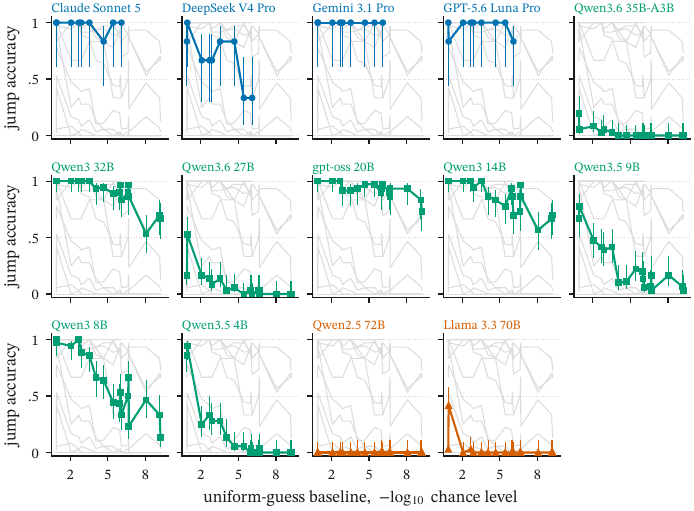}
\caption{Jump accuracy against the certified uniform-table chance level, one panel per
model with the other thirteen curves in grey. The primary cover story is used. Shared
open-model configurations pool $6+30=36$ original and scaled answers; additional
configurations have $30$, and API configurations have six. Bars are nominal Wilson
$95\%$ intervals. The horizontal axis describes the uniform-guess baseline.}
\label{fig:difficulty-full}
\end{figure}

\begin{figure}[tbp]
\centering
\includegraphics[width=\textwidth]{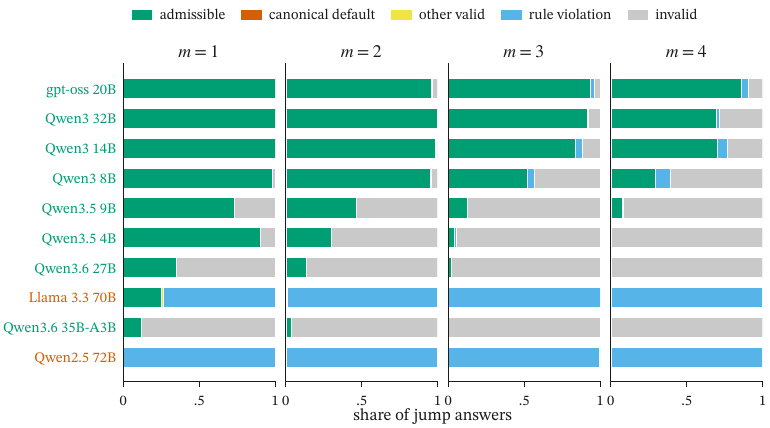}
\caption{Constrained outcomes in the scaled family, grouped by chain length $m$.
Each model contributes $60$, $120$, $240$, and $120$ calls at $m=1,2,3,4$.
Models are ordered by overall jump accuracy. The canonical outcome is absent
from every bar; the gray invalid category includes format and termination failures.}
\label{fig:taxonomy}
\end{figure}

\begin{figure}[tbp]
\centering
\includegraphics[width=0.82\textwidth]{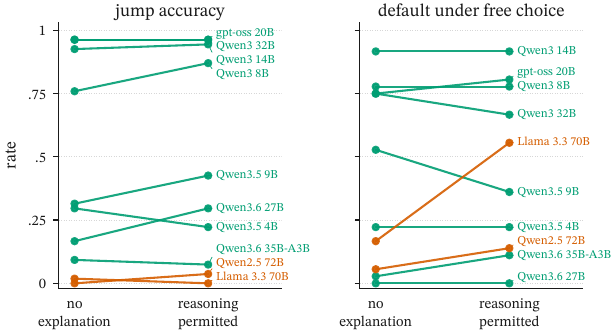}
\caption{The first cover story with explanations suppressed or permitted before
the final answer. Each condition has $54$ jump and $36$ calibration calls
per open model, and every call remains in the denominators. This comparison
changes the requested output format; it does not disable reasoning traces
already emitted by the model's chat template.}
\label{fig:scratchpad}
\end{figure}

\begin{figure}[tbp]
\centering
\includegraphics[width=\textwidth]{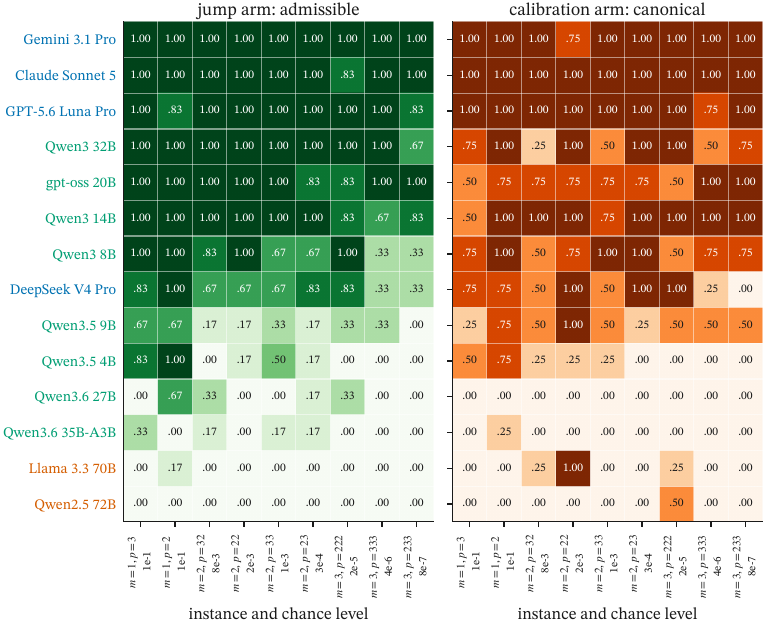}
\caption{The original pointed-chain instances under the primary cover story.
Each model--instance cell contains six jump and four calibration calls.
The panels show admissible jump fractions and canonical calibration
fractions, including failures in both denominators.}
\label{fig:heatmap}
\end{figure}

\begin{figure}[tbp]
\centering
\begin{subfigure}[t]{0.60\textwidth}
\centering
\includegraphics[width=\textwidth]{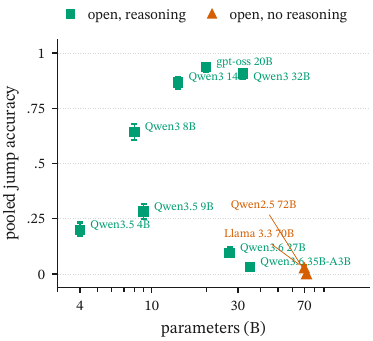}
\caption{Jump accuracy against nominal total parameter count. Each open model
pools $108$ original-chain and $540$ scaled calls.}
\label{fig:size}
\end{subfigure}\par\vspace{8pt}
\begin{subfigure}[t]{\textwidth}
\centering
\includegraphics[width=\textwidth]{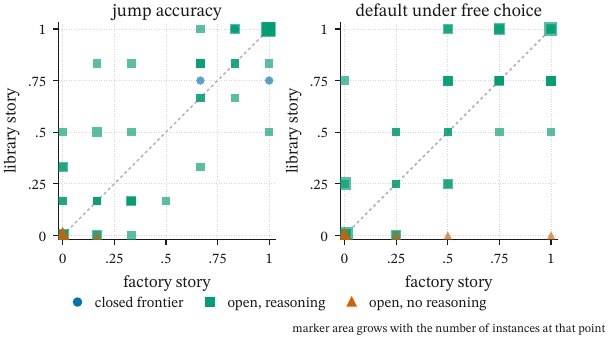}
\caption{The two stories on shared model--instance pairs. Open-model jump and
calibration cells contain six and four calls per story; API jump cells
contain six first-story and four second-story calls.}
\label{fig:rendering}
\end{subfigure}
\caption{Descriptive model-size and wording comparisons. Colors and symbols denote
the evaluated model groups. In the wording panels, marker area grows with
the number of instances sharing a location within one model; the paired calibration panels
cover the open models.}
\label{fig:appendix-pair}
\end{figure}

\end{document}